\documentclass{ecai} 

\usepackage{latexsym}
\usepackage{amssymb}
\usepackage{amsmath}
\usepackage{amsthm}
\usepackage{booktabs}
\usepackage{graphicx}
\usepackage{color}

\usepackage{hyperref}
\usepackage[dvipsnames, svgnames,x11names]{xcolor}
\usepackage[ruled,vlined,linesnumbered]{algorithm2e} 
\usepackage{tikz}
\usepackage{verbatim}
\usepackage{pgfplots}
\usepackage{caption}
\usepackage{subcaption}
\usepackage{mathtools}

\usepackage[inline]{enumitem}

\usetikzlibrary{patterns}
\usepgfplotslibrary{groupplots, dateplot, fillbetween}
\pgfplotsset{compat=1.18} 
\tikzset{dot/.style = {circle, fill, minimum size=#1,inner sep=0pt, outer sep=0pt, fill, circle},dot/.default = 6pt}
\tikzset{dot2/.style = {circle, fill, color=black!40,minimum size=6pt,inner sep=0pt, outer sep=0pt, fill, circle}}
\tikzstyle{a}=[->,>=stealth]
\usetikzlibrary{positioning}

\newtheorem{theorem}{Theorem}

\newtheorem{proposition}[theorem]{Proposition}

\newtheorem{definition}{Definition}
\newtheorem{assumption}{Assumption}

\newcommand{\BibTeX}{B\kern-.05em{\sc i\kern-.025em b}\kern-.08em\TeX}

\begin{document}


\begin{frontmatter}


\paperid{3982} 


\title{Understanding Action Effects through Instrumental Empowerment in Multi-Agent Reinforcement Learning}


\author[A]{\fnms{Ardian}~\snm{Selmonaj}}
\author[A]{\fnms{Miroslav}~\snm{Štrupl}}
\author[A]{\fnms{Oleg}~\snm{Szehr}} 
\author[A]{\fnms{Alessandro}~\snm{Antonucci}} 

\address[A]{Istituto Dalle Molle di Studi sull'Intelligenza Artificiale (IDSIA), USI-SUPSI, Lugano, Switzerland}


\begin{abstract}
To reliably deploy \emph{Multi-Agent Reinforcement Learning} (MARL) systems, it is crucial to understand individual agent behaviors. While prior work typically evaluates overall team performance based on explicit reward signals, it is unclear how to infer agent contributions in the absence of any value feedback. In this work, we investigate whether meaningful insights into agent behaviors can be extracted solely by analyzing the policy distribution. Inspired by the phenomenon that intelligent agents tend to pursue convergent instrumental values, we introduce \emph{Intended Cooperation Values} (ICVs), a method based on information-theoretic Shapley values for quantifying each agent’s causal influence on their co-players' instrumental empowerment. Specifically, ICVs measure an agent's action effect on its teammates’ policies by assessing their decision (un)certainty and preference alignment. By analyzing action effects on policies and value functions across cooperative and competitive MARL tasks, our method identifies which agent behaviors are beneficial to team success, either by fostering deterministic decisions or by preserving flexibility for future action choices, while also revealing the extent to which agents adopt similar or diverse strategies. Our proposed method offers novel insights into cooperation dynamics and enhances explainability in MARL systems.
\end{abstract}

\end{frontmatter}

\section{Introduction}\label{sec:intro}

The rise of \emph{Multi-Agent Reinforcement Learning}~(MARL) has led to many real-world problems being reformulated as games. MARL offers the greatest benefits when combined with powerful function approximators, such as deep neural networks. Beyond robust performance, it is crucial to understand these models by extensively testing and analyzing their behavior in simulation environments to ensure safe deployment in real-world systems. This need has driven attention toward \emph{post-hoc} explanation methods, which treat models as black boxes without constraining their complexity. A game-theoretic method in this category is the \emph{Shapley Value}~(SV), which was originally designed to fairly distribute the overall payoff among players in a cooperative game~\citep{shapley_fair_payoff}. It has since been adapted as a feature attribution technique to explain why specific predictions are made in a deep learning model~\citep{shapley_ml_7}.

In MARL, the reward function is fundamental in shaping agent behavior, with Schmidhuber's artificial curiosity~\citep{schmidi_curiosity} inspiring the design of intrinsic rewards to foster cooperation. As the majority of MARL algorithms (see, e.g.,~\citep{marl_benchmarks}) rely on learned value functions to estimate expected rewards in a given state, the SV naturally emerges as a credit assignment method in fully cooperative MARL scenarios. Conversely, in the absence of both value functions and reward signals, a critical question arises of how, and more importantly \emph{which}, credit should be allocated, irrespective of the game type (cooperative or competitive). In this setting of missing explicit value feedback, we investigate whether SVs can be adapted to extract meaningful information about agent behavior on performance, i.e., information that remains consistent with the underlying value function, while assuming access \emph{only} to the policy function. 

To identify a meaningful candidate for credit, we draw on the statement by~\citet{act_many_options}: "I shall act always so as to increase the total number of choices," thereby advocating for individual empowerment. This aligns with the \emph{instrumental convergence thesis}~\citep{superintelligence_instr_conv}, which posits that several instrumental values generally support task success and are thus pursued by many intelligent agents. Given theoretical~\citep{instr_conv_proof2_power_seeking} and empirical~\citep{altruism_paper, empowerment_maxreward} evidence that access to and control over a wide range of future states (choices) serves as convergently instrumental objective during \emph{training}, we question whether the same holds true during \emph{inference} in multi-agent settings. Concretely, is providing agents with more choices truly beneficial for task success?

\begin{figure}[htb!]
    \centering
    \hspace{0.2cm}
    \begin{subfigure}{0.4\columnwidth}
        \centering
        \includegraphics[width=\linewidth]{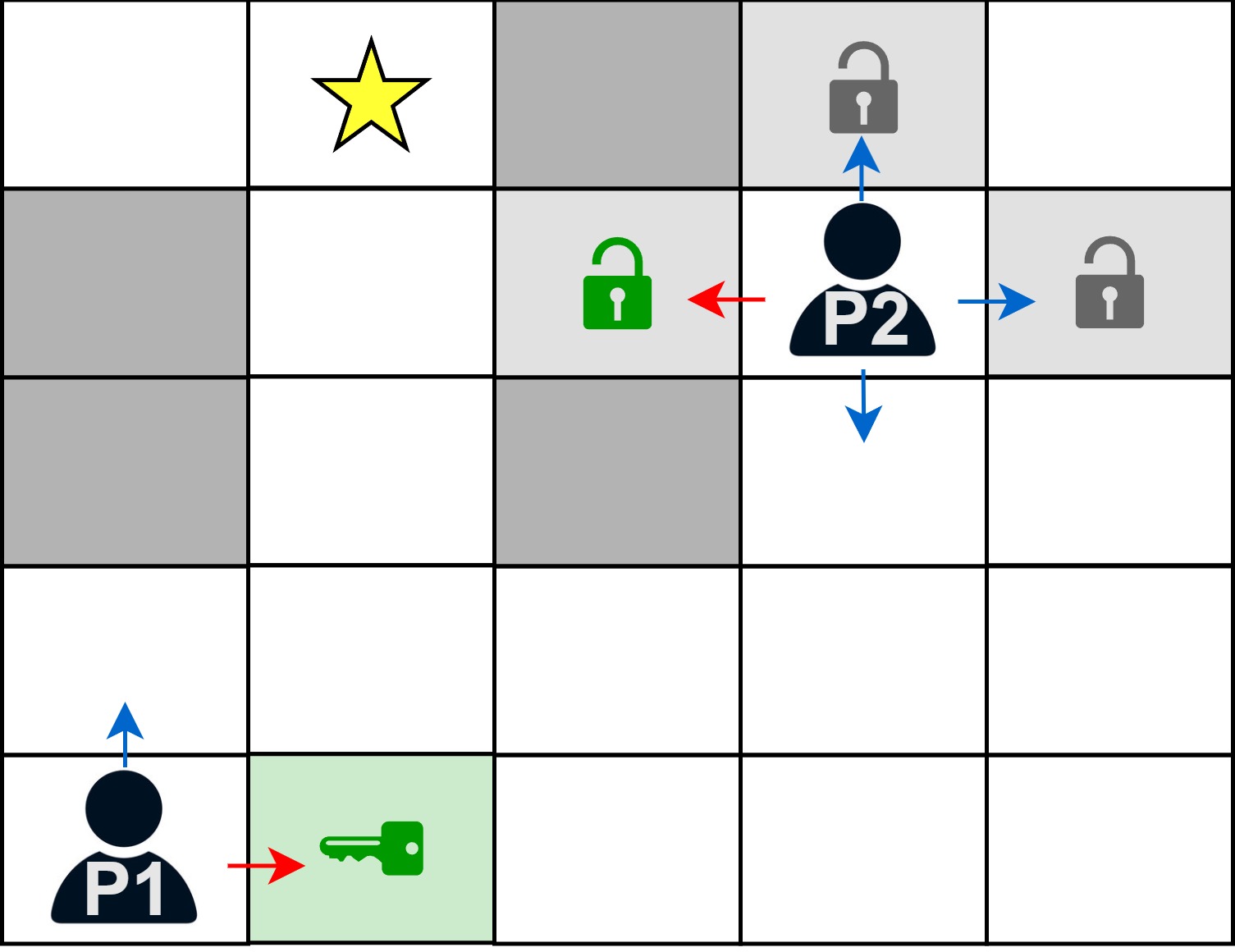}
    \end{subfigure}
    \hspace{0.3cm}
    \begin{subfigure}{0.5\columnwidth}
        \centering
        \begin{tikzpicture}
            \begin{axis}[
                width=\columnwidth+0.3cm,
                height=2.35cm,
                enlarge x limits=0.2,
                xtick style={/pgfplots/major tick length=0pt},
                xtick={1, 2, 3, 4},
                xticklabels={up, down, left, right},
                ybar,
                bar width=3.5pt,
                ymin=0, ymax=1,
                ytick distance=1,
                ticklabel style={font=\footnotesize},
                xticklabel style={yshift=4pt},
            ] 

            \addplot[fill=NavyBlue , draw= NavyBlue]
                coordinates {(1,0.1) (2,0.55) (3,0.3) (4, 0.05)};
            
            \end{axis}
        \end{tikzpicture}
        \begin{tikzpicture}
            \begin{axis}[
                width=\columnwidth+0.3cm,
                height=2.35cm,
                enlarge x limits=0.2,
                xtick style={/pgfplots/major tick length=0pt},
                xtick={1, 2, 3, 4},
                xticklabels={up, down, left, right},
                ybar,
                bar width=3.5pt,
                ymin=0, ymax=1,
                ytick distance=1,
                ticklabel style={font=\footnotesize},
                xticklabel style={yshift=4pt},
            ] 

            \addplot[fill=NavyBlue , draw= NavyBlue]
                coordinates {(1,0.05) (2,0.1) (3,0.8) (4, 0.05)};
            
            \end{axis}
        \end{tikzpicture}
        \vspace{-0.1cm}
    \end{subfigure}
\vspace{0.2cm}
\caption{Left: a toy environment where both players must reach the target (star) to earn a shared reward, while gray cells block the path. Blue and red arrows indicate potential moves (choices). Right: probabilities (y-axis) over actions (x-axis) of P2 before (top) and after (bottom) P1 steps on the key cell and opens the green lock.}
\label{fig:toy_env_concept}
\vspace{0.7cm}
\end{figure}
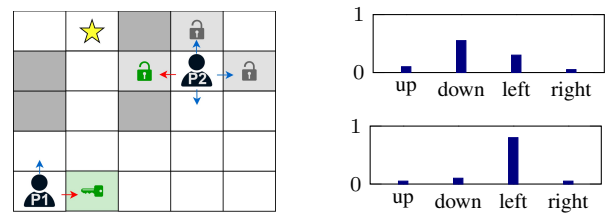

Contrary to the above findings, we argue that \emph{exploiting certain situations with high certainty} may be more instrumentally valuable than preserving many choices. We illustrate this concept in Fig.~\ref{fig:toy_env_concept}, where two agents must reach a target (star) to receive a shared reward of one. The histograms on the right depict the probability distributions over movement actions for P2 before (top) and after (bottom) P1 acted. The closed green lock cell makes P2 more likely to move downward, while still assigning some probability to the leftward action in anticipation that P1 might open the lock. Consider the case where the key opens the green lock only and P1 steps on it. This action causally affects P2’s action space by unlocking the left path. As a result, P2 not only gets more movement options, but also becomes more certain in exploiting the open path by assigning more weight toward the left action. This highlights how P2's (increased) success is directly tied to P1’s altruistic behavior, which would correspond to an increased value under an optimal value function. In contrast, if the key were to unlock the gray locks, it would provide additional movement options that are clearly not beneficial for achieving its goal. Even if the game involves simultaneous decision-making and P2 chooses to go down, P1 would merit credit for his \emph{intention} to open the path and contribute to P2’s success. In addition, we examine the impact on preference alignment between agents. To do so, we perform a state exchange intervention by placing P1 in P2’s position and v.v. to assess whether agents exhibit coordinated consensus through mutual anticipation of similar strategies. Credit is attributed to actions that enhance alignment and promote consensus.

Motivated by this, we introduce an action attribution method called \emph{Intended Cooperation Value}~(ICV), which quantifies how cooperatively an agent behaves by measuring the causal effect of its actions on the instrumental empowerment of its teammates, using only access to their policies. By adapting causal, information-theoretic SVs, ICV uses the policy entropy to assign credit to an agent based on its intended contribution to reducing decision uncertainty and fostering preference alignment among co-players. Empirical evaluations across different settings demonstrate that ICV reliably attributes credit to behaviors that increase the likelihood of task success. Our approach provides a deeper understanding of agent interactions, particularly in scenarios where success is not explicitly measurable.

\section{Related Work}\label{sec:rel_work}

There has been considerable work on understanding and explaining \emph{Reinforcement Learning}~(RL), primarily aimed at answering \emph{why} specific actions were taken. Existing approaches include replacing the policy network with transparent models, such as decision trees~\citep{dec-tree2}, or incorporating auxiliary reasoning mechanisms like counterfactual analysis~\citep{exp_counterfact}, structural causal models~\citep{exp_causal_RL}, or interpretable reward components~\citep{exp_rew_components}. The SV has been used as a feature attribution method in \emph{Machine Learning}~(ML)~\citep{shapley_ml3, shapley_ml4}, in single-agent RL~\citep{shapley_rl1, shapley_rl2}, and in MARL to assess agent contributions based on rewards~\citep{shapley_marl_importance}. Beyond explanatory purposes, SVs have also been adapted as training mechanisms in MARL, mainly to address the credit assignment problem~\citep{shap_train_marl1, shap_train_marl2, shap_train_marl3}. In parallel, various information-theoretic intrinsic rewards have been proposed to foster multi-agent cooperation~\citep{intr_rew_causal_MARL, intr_rew_diversity}.

So far, no work has examined whether causal influences on agent policies, measured via information-theoretic SVs, correspond to changes in the value function. We are the first to explore if analyzing \emph{solely} the policy provides meaningful insights into an agent’s behavior and its impact on co-players' decisions, and if these insights align with overall team performance to assign a reward-free credit.

\section{Preliminaries}\label{sec:preliminaries}
Let $(\Omega,\mathcal{F},\Pr)$ be a probability space. Let $X:\Omega \rightarrow \mathcal{X}$ be a discrete \emph{Random Variable}~(RV) taking value $x$ in a set $\mathcal{X}$, where $p(x) = \Pr(X = x)$ is the \emph{Probability Mass Function}~(PMF) and $\mathcal{P}(\mathcal{X})$ is the set of probability distributions over $\mathcal{X}$. We restrict our attention to finite sets $\mathcal{X}$. The entropy of $X$ is given by $H(X)=-\sum_{x\in \mathcal{X}} p(x)\log p(x)$. For two RVs $X$ and $Y$, the mutual information is a symmetrical function and defined as $\mathrm{I}(X;Y) = H(X) - H(X \mid Y)=  H(Y) - H(Y \mid X)$. For two PMFs $p$ and $q$, the Kullback–Leibler divergence is $D_{KL}(p||q)=\sum_{x\in \mathcal{X}}p(x) \log \frac{p(x)}{q(x)}$.

\subsection{Multi-Agent Reinforcement Learning}\label{sec:preliminaries_RL}

Interactions among $n \in \mathbb{N}$ concurrently acting agents are typically modeled using either a \emph{Markov Game} (MG)~\citep{markov_game_def} or an \emph{Extensive-Form Game} (EFG)~\citep{ext_form_games}. While in EFGs actions are selected and executed sequentially, in MGs actions are selected simultaneously but their execution may be either parallel or sequential, depending on whether the environment supports parallelization. If not, actions are executed based on a predefined or randomized agent ordering.

\begin{assumption}
Without loss of generality, we focus on MGs, as simultaneous decision-making is commonly used in MARL systems~\citep{marl_book}.
\label{ass:mg}
\end{assumption}

\begin{definition} A Markov Game is defined by the tuple
    $$\mathcal{G}_{m} \coloneqq (\mathcal{N}, \mathcal{S}, \{\mathcal{A}^i\}_{i=1}^n, P, \rho, \{R^i\}_{i=1}^n, \gamma),$$
where $\mathcal{N} = \{1, \dots, n\}$, with $n\in \mathbb{N}$, is the set of agents, $\mathcal{S}$ is the set of admissible states, $\mathcal{A}^i$ is the set of actions available to agent $i \in \mathcal{N}$, with $\boldsymbol{\mathcal{A}}=\prod_{i=1}^n \mathcal{A}^i$ being the joint action set, $P: \mathcal{S} \times \boldsymbol{\mathcal{A}} \to \mathcal{P}(\mathcal{S})$ is the state transition probability kernel, $\rho \in \mathcal{P}(\mathcal{S})$ is the initial state distribution, $R^i: \mathcal{S} \times \boldsymbol{\mathcal{A}} \times \mathcal{S} \to \mathbb{R}$ is the reward function of agent $i$, and $\gamma \in [0, 1)$ is the discount factor controlling the effect of future rewards. At each time step $t \in \mathbb{N}$, let the RV $S_t$ describe the state of the MG, and the RV $A_t^i$ denote the action taken by agent $i$. Starting from an initial state $s_0 \sim \rho$, all agents observe the current state $s_t \in \mathcal{S}$ and simultaneously select an action $a_t^i \in \mathcal{A}^i$ according to their individual policies $\pi^i(a \mid s_t)$, forming a joint action $\boldsymbol{a}_t = (a_t^1, \dots, a_t^n) \in \boldsymbol{\mathcal{A}}$ drawn from the joint policy $\boldsymbol{\pi}( \boldsymbol{a}\mid s_t) = \prod_{i=1}^n \pi^i(a\mid s_t)$. The MG then transitions to a new state $s_{t+1}$ according to $P(s_{t+1} \mid s_t, \boldsymbol{a}_t)$, and each agent $i$ receives a reward $R^i(s_t, \boldsymbol{a}_t, s_{t+1})$. The performance of each agent $i$ can be measured by means of a value function defined as:
\begin{equation}
    V^{i}(s) \coloneqq \mathbb{E}_{\boldsymbol{a}_t \sim \boldsymbol{\pi}, s_{t+1}\sim P} \left[ \sum_{t=0}^{\infty} \gamma^t R^i(s_t, \boldsymbol{a}_t, s_{t+1})  \mid s_0 = s\right].
    \label{eq:value_function}
\end{equation}
\label{def:mg}
\end{definition}

\begin{assumption}
We consider pre-trained \emph{Actor-Critic}~\citep{actor_critic} models under the MG framework using \emph{Centralized Training with Decentralized Execution} (CTDE) \citep{ctde}. Further, we consider stochastic policies, which are generally more robust in real-world-like environments compared to deterministic policies~\citep{robust_policies}.
\label{ass:ctde}
\end{assumption}

\begin{definition}
Multi-agent systems (game types) can be categorized as~\citep{schelling_games}: 
\begin{enumerate*}[label=\textit{\roman*)}]
    \item \emph{Cooperative}: all agents $i\in \mathcal{N}$ share a global reward $R=R^i$.
    \item \emph{Competitive}: $\sum_{i=1}^n R^i=0$ for any state transition (zero-sum MG).
    \item \emph{Mixed-motive}: having individual rewards that induce cooperative and competitive motivations.
\end{enumerate*}
\label{def:game_categories}
\end{definition}

Assumption~\ref{ass:mg} implies that we consider domains with full state observability, though our approach can trivially be extended to partial observability and sequential decisions. Assumption~\ref{ass:ctde} implies that agents act with individual policies (actors), and in fully cooperative settings, they share a global value function $V(s)$, also called centralized critic~\citep{ctde_overview}. From Definition~\ref{def:game_categories}, we emphasize the flexibility of MGs to represent all game types~\citep{cooperation_review}. We use $\boldsymbol{-i}$ to denote the set of all agent indices excluding agent $i$. We further note that the terms environment, game, and MG are used interchangeably.

\subsection{Shapley Value}\label{sec:shapley_value}
We build on prior work applying SVs for feature attribution in single‐agent RL and ML models. Thus, let $Z$ represent the input RV.

\begin{definition}
    Let $\mathcal{M} = \{1,\dots,m\}$ be the set of features and $\mathcal{Z}=\prod_{i=1}^{m} \mathcal{Z}_i$ the input space, so that any input $z$ can be represented as an ordered set $z=\{z_i \mid z_i \in \mathcal{Z}_i\}_{i=1}^{m}$, with $p(z)$ being the data distribution. Then, any subset $\mathcal{K} \subseteq \mathcal{M}$ yields a partial observation $z_{\mathcal{K}} = \{z_i\}_{i \in \mathcal{K}}$, where $\bar{\mathcal{K}} = \mathcal{M} \setminus \mathcal{K}$ is its complement.
    \label{def:input_features}
\end{definition}

\begin{definition}
    The Shapley Value~\citep{shapley_original} is the unique method that satisfies the desirable properties of efficiency, symmetry, additivity, and null player (see Appendix~\ref{app:shapley_axioms}~\cite{fullpaper}).
    It allocates credit by computing the average marginal contribution $\Delta v(\cdot)$ of player (or feature) $i$ to the outcome of a cooperative game (or prediction) over all coalitions $\mathcal{K} \subseteq \mathcal{M}$ using a characteristic function $v: 2^{\mathcal{M}} \rightarrow \mathbb{R}$, as follows:
    $$\phi_i(v) = \sum_{\mathcal{K} \subseteq \mathcal{M} \setminus \{i\}} \frac{|\mathcal{K}|! \, (|\mathcal{M}| - |\mathcal{K}| - 1)!}{|\mathcal{M}|!} \underbrace{ \left[ v(\mathcal{K} \cup \{i\}) - v(\mathcal{K}) \right] }_{\Delta v_i(\cdot)}\,.$$
    \label{def:shapley_def_standard}
\end{definition}

\section{Method}\label{sec:method}
We aim for a method applicable during inference without requiring a separately trained model or modifications to the underlying MARL architecture. We introduce our approach as follows:
\begin{enumerate*}[label=\textit{\roman*)}]
    \item a framework for modeling simultaneous action selection with sequential action execution (Sec.~\ref{sec:game_model}),
    \item an adaptation of SVs to isolate an agent’s causal effect and measure its marginal contribution during gameplay (Sec.~\ref{sec:action_attribution}),
    \item characteristic functions for quantifying the instrumental value of an agent’s action on its co-players (Sec.~\ref{sec:charact_funcs}).
\end{enumerate*}

\subsection{Game Model}\label{sec:game_model}

The ordinary MG framework makes it hard to infer individual action effects on others. For instance, in Fig.~\ref{fig:toy_env_concept}, imagine an agent P3 is positioned below the green lock instead of the wall and P3 could also step onto the lock cell when it gets opened by P1. By analyzing only $S_t$ and $S_{t+1}$, one cannot determine individual contributions of P1 and P3 on P2’s decision-making. To overcome this limitation, we propose a new game model, the \emph{Sequential Value Markov Game}~(SVMG).
\begin{definition}
    An SVMG is defined by the tuple
    $$\mathcal{G}_{s} \coloneqq (\mathcal{N}, \mathcal{S}, \{\mathcal{A}^i\}_{i=1}^n, \{P^i\}_{i=1}^n, \mathcal{D}, \rho),$$
where $\mathcal{N}, \mathcal{S}, \{\mathcal{A}^i\}_{i=1}^n$, and $\rho$ are the same variables as for $\mathcal{G}_m$ in Def.~\ref{def:mg}. A transition kernel $P^i:\mathcal{S} \times \mathcal{A}^i \rightarrow \mathcal{P}(\mathcal{S})$ considers only one agent action. $\mathcal{D} \in \mathcal{P}(\Sigma)$ specifies a distribution over the set of all $n!$ permutations $\Sigma \coloneqq \{\sigma: \mathcal{N} \hookrightarrow \mathcal{N}\}$.
\label{def:svmg}
\end{definition}

Consider Fig.~\ref{fig:MG_to_SVMG} to better elaborate on the SVMG framework.
Without loss of generality we assume that $\mathcal{S}  = \prod_{i=1}^n \mathcal{S}_i$ so that any state RV $S:\Omega \rightarrow \mathcal{S}$ can be decomposed into states of individual agents $S=(S^1,S^2,\ldots,S^n)$, where $S^i:\Omega\rightarrow \mathcal{S}^i$.
Then, actions are executed sequentially in $\mathcal{G}_s$ according to an ordering $\sigma \sim \mathcal{D}$, where $\mathcal{D}$ is either uniform or deterministic.
While $i\in \mathcal{N}$ is an agent index, let $k\in \mathcal{N}$ be a sub-step index iterating over $\sigma$, where $\sigma(k)$ denotes the agent assigned to the $k$-th position in the ordering. Thus, agent $i$ acts at sub-step $k$ if and only if $i=\sigma(k)$.
\begin{definition}  
    We define \emph{intermediate} state RVs $S_{t,(k)}$ and $S_t$ for $t$ and $k$ in the following way. Let $S_0 = S_{0,(0)} \sim \rho$. For each $t$, we construct $S_{t,(k)}$ such that $(S_{t,(k)}|S_{t,(k-1)}, A_t^{\sigma(k)})\sim P^{\sigma(k)}$, and $S_{t+1} = S_{t,(n)}$. Each  kernel $P^i$ has the property that it updates only the state of player $i$ and the states of other agents stay fixed, i.e.,  $P^i(s_k|s_{k-1},a^i) = 0$ for all $j\neq i$ when $s_k^j \neq s_{k-1}^j$. At each decision step $t$, a new order $\sigma$ is sampled $iid$ from $\mathcal{D}$.
    \label{def:inter_state}
\end{definition}

\begin{figure}[htb!]
    \centering
    \begin{tikzpicture}[scale=1]

        \node[draw,circle,minimum size=0.9cm,inner sep=0pt, fill=SteelBlue!70] (s1) at (0.5,5.4) {$S_t$};
        \node[draw,circle,minimum size=0.9cm,inner sep=0pt, fill=SteelBlue!70] (s2) at (7.6,5.4) {$S_{t+1}$};
        \node[draw,circle,minimum size=0.9cm,inner sep=0pt, fill=Dandelion!70] (a00) at (4.025,5.4) {$\boldsymbol{A}_{t}$};
        
        \node[draw,circle,minimum size=0.9cm,inner sep=0pt, fill=SteelBlue!40] (z0) at (0.5,4) {$S_{t,(0)}$};
        \node[draw,circle,minimum size=0.9cm,inner sep=0pt, fill=SteelBlue!40] (z1) at (2.8,4) {$S_{t, (1)}$};
        \node[draw,circle,minimum size=0.9cm,inner sep=0pt, fill=SteelBlue!40] (z2) at (5.2,4) {$S_{t, (2)}$};
        \node[draw,circle,minimum size=0.9cm,inner sep=0pt, fill=SteelBlue!40] (z3) at (7.6,4) {$S_{t, (3)}$};
        \node[draw,circle,minimum size=0.9cm,inner sep=0pt, fill=Dandelion!40] (a1) at (1.65,4) {$A_t^{\sigma(1)}$};
        \node[draw,circle,minimum size=0.9cm,inner sep=0pt, fill=Dandelion!40] (a2) at (4,4) {$A_t^{\sigma(2)}$};
        \node[draw,circle,minimum size=0.9cm,inner sep=0pt, fill=Dandelion!40] (a3) at (6.4,4) {$A_t^{\sigma(3)}$};
        
        \draw[a] (z0) -- (a1);
        \draw[a] (a1) -- (z1);
        \draw[a, color=gray!50] (z1) -- (a2);
        \draw[a] (a2) -- (z2);
        \draw[a, color=gray!50] (z2) -- (a3);
        \draw[a] (a3) -- (z3);
        \draw[a] (z0) to[bend right=50] (z1);
        \draw[a] (z1) to[bend right=50] (z2);
        \draw[a] (z2) to[bend right=50] (z3);
        \draw[<->, line width=1.2pt, color=BrickRed] (s1) -- (z0);
        \draw[<->, line width=1.2pt, color=BrickRed] (s2) -- (z3);
        \draw[a] (s1) -- (a00);
        \draw[a] (a00) -- (s2);
        \draw[a] (s1) to[bend left=15] (s2);
        \draw[dashed, <->] (a00) -- (a1);
        \draw[dashed, <->] (a00) -- (a2);
        \draw[dashed, <->] (a00) -- (a3);
        
        \draw[dashed, <->] (z0) -- (0.5,3);
        \draw[dashed, <->] (z1) -- (2.8,3);
        \draw[dashed, <->] (z2) -- (5.2,3);
        \draw[dashed, <->] (z3) -- (7.6,3);

        \draw[fill=SteelBlue!40] (0,0) rectangle (1,3);
        \node[draw,circle,minimum size=0.9cm,inner sep=0pt, fill=SteelBlue!20] (z11) at (0.5,2.5) {$S^{\sigma(1)}_{t}$};
        \node[draw,circle,minimum size=0.9cm,inner sep=0pt, fill=SteelBlue!20] (z12) at (0.5,1.5) {$S^{\sigma(2)}_{t}$};
        \node[draw,circle,minimum size=0.9cm,inner sep=0pt, fill=SteelBlue!20] (z13) at (0.5,0.5) {$S^{\sigma(3)}_{t}$};


        \draw[fill=SteelBlue!40] (2.3,0) rectangle (3.3,3);
        
        \node[draw,circle,minimum size=0.9cm,inner sep=0pt, fill=SteelBlue!20] (z21) at (2.8,2.5) {$S^{\sigma(1)}_{t+1}$};
        \node[draw,circle,minimum size=0.9cm,inner sep=0pt, fill=SteelBlue!20] (z22) at (2.8,1.5) {$S^{\sigma(2)}_{t}$};
        \node[draw,circle,minimum size=0.9cm,inner sep=0pt, fill=SteelBlue!20] (z23) at (2.8,0.5) {$S^{\sigma(3)}_{t}$};

        \draw[fill=SteelBlue!40] (4.7,0) rectangle (5.7,3);

        \node[draw,circle,minimum size=0.9cm,inner sep=0pt, fill=SteelBlue!20] (z31) at (5.2,2.5) {$S^{\sigma(1)}_{t+1}$};
        \node[draw,circle,minimum size=0.9cm,inner sep=0pt, fill=SteelBlue!20] (z32) at (5.2,1.5) {$S^{\sigma(2)}_{t+1}$};
        \node[draw,circle,minimum size=0.9cm,inner sep=0pt, fill=SteelBlue!20] (z33) at (5.2,0.5) {$S^{\sigma(3)}_{t}$};

        \draw[fill=SteelBlue!40] (7.1,0) rectangle (8.1,3);

        \node[draw,circle,minimum size=0.9cm,inner sep=0pt, fill=SteelBlue!20] (z31) at (7.6,2.5) {$S^{\sigma(1)}_{t+1}$};
        \node[draw,circle,minimum size=0.9cm,inner sep=0pt, fill=SteelBlue!20] (z32) at (7.6,1.5) {$S^{\sigma(2)}_{t+1}$};
        \node[draw,circle,minimum size=0.9cm,inner sep=0pt, fill=SteelBlue!20] (z33) at (7.6,0.5) {$S^{\sigma(3)}_{t+1}$};
    \end{tikzpicture}
\vspace{0.3cm}
\caption{Transformation of MG (upper-level, saturated colors) to SVMG (low-level, faded colors) for $n=3$ agents, following an execution order $\sigma$. Red lines between states mean correspondences, while dashed lines mean decomposition into individual components.}
\label{fig:MG_to_SVMG}
\vspace{0.8cm}
\end{figure}
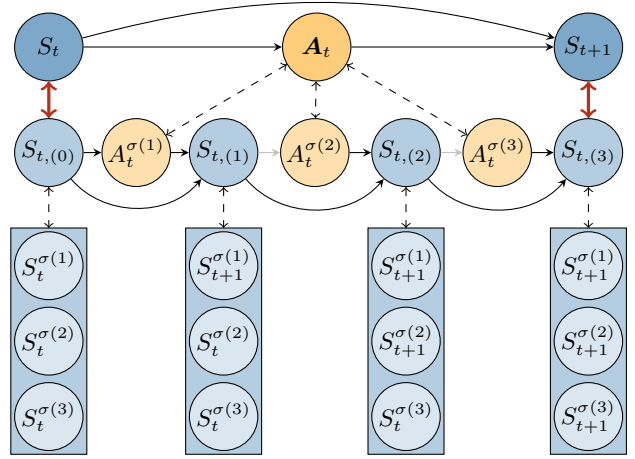

By Definition~\ref{def:inter_state}, agent $i$ got its local state $S_{t,(k)}^i$ is updated at sub-step $k$ if it belongs to the first $k$ actors, i.e.,
\begin{equation}
    S_{t,(k)}^i=
    \begin{cases}
        S_{t+1}^i,& \sigma^{-1}(i) \leq k\\
        S_{t}^i ,&\sigma^{-1}(i) > k
    \end{cases}
    \label{eq:inter_state_udpate_explicit}
\end{equation}

\begin{assumption}
    Def.~\ref{def:inter_state} implies that we assume $\mathcal{G}_s$ to be locally Markovian, i.e., for each step $k$, the next intermediate state $s_{t,(k)}$ depends only on the immediately preceding state $s_{t,(k-1)}$ and the action $a_t^i$ given that $i$ is the currently acting player:
    \begin{align*}
    &\Pr(S_{t,(k)} = s_{t,(k)} \mid \{S_{t, (j)}=s_{t, (j)}\}_{j<k}, \{A_t^j=a_t^j\}_{j=1}^n, \sigma(k)=i) 
    \\& = P^i(s_{t,(k)} \mid s_{t, (k-1)}, a_t^i).
    \end{align*}    
    \label{ass:local_markov}
\end{assumption}

\begin{assumption}
    We assume the underlying MG can be expressed by the SVMG by fulfilling the following conditions:
    \begin{enumerate}[label=\textit{\roman*)}]
    \item The environment possess an internal ordering of agents, according to which actions are executed sequentially.
    \item The joint kernel $P$ of $\mathcal{G}_m$ is decomposable by the per sub-step kernels $P^{\sigma(k)}$ of $\mathcal{G}_s$, such that by marginalizing over all orderings $\sigma \in \Sigma$ and intermediate states $\mathbf{s}=(s_{t,(1)}, \dots, s_{t,(n-1)})$, we recover the joint kernel satisfying $s_{t,(0)} = s_t$ and $s_{t,(n)} = s_{t+1}$:
        \begin{equation}
            P(s_{t+1} \mid s_t, \boldsymbol{a}_t) =  \mathbb{E}_{\sigma} \sum_{\mathbf{s} \in \mathcal{S}} \prod_{k=1}^{n-1} P^{\sigma(k)}\left(s_{t, (k)} \mid s_{t,(k-1)}, a_t^{\sigma(k)}\right)
        \end{equation}
    \end{enumerate}
    \label{ass:mg_as_svmg}
\end{assumption}

While Assumption~\ref{ass:mg_as_svmg} may hold only in simple environments, it enables our framework to be applied \emph{online} during execution. In such cases, intermediate states $S_{t,(k)}$ can be accessed while actions are being processed. In contrast, for more complex games that support parallel execution and thereby violate Assumption~\ref{ass:mg_as_svmg}, our SVMG framework remains applicable \emph{offline}. By storing the global state trajectory $\{S_t\}_{t=0}^{T}$ during inference, the intermediate sequences $\{S_{t,(k)}\}_{k=1}^{n-1}$ can be reconstructed by Eq.~\eqref{eq:inter_state_udpate_explicit} in a post-processing step with $\sigma$ injected externally, i.e., $\sigma \sim \mathcal{D}$.

\paragraph{Causality.} $\mathcal{G}_s$ in Fig.~\ref{fig:MG_to_SVMG} represent causal \emph{Directed Acyclic Graphs} (DAGs), where edges like $S_{t,(k)} \rightarrow A^{\sigma(k+1)}_t$ indicate direct causal relationships. However, for $k>1$, $S_{t,(k)}$ may \emph{not} directly cause $A^{\sigma(k+1)}_t$, since joint actions $\boldsymbol{A}_t$ are pre-determined at $S_t = S_{t,(0)}$, and agents do not react to all intermediate states. To account for this, we perform an \emph{intervention}~\citep{causality_def} by fixing $A^{\sigma(k)}_t$ to $a^{\sigma(k)}_t$ as $\mathrm{do}(A^{\sigma(k)}_t = a^{\sigma(k)}_t)$, independently of its parents. Graphically, this corresponds to removing incoming arcs to $A^{\sigma(k)}_t$, yielding a \emph{surgery} marked by grayish arcs in Fig.~\ref{fig:MG_to_SVMG}. Viewing $\mathcal{G}_s$ as a causal DAG enables post-interventional queries to analyze local inter-agent influences, as similarly done by~\citet{intr_rew_causal_MARL}.

\subsection{Action Attribution}\label{sec:action_attribution}
We present our adaptation of ordinary SVs to isolate an agent’s causal effect and draw inspiration from \emph{Causal Shapley Values}, which have been applied to ML predictions. Based on Def.~\ref{def:input_features} and~\ref{def:shapley_def_standard}, \citet{shapley_causal_intervene} modify the data distribution by intervening on coalition features and propose $v=p(z_{\bar{\mathcal{K}}}\mid \text{do}(Z_{\mathcal{K}}=z_{\mathcal{K}}))$. In contrast, \citet{shapley_causal_asymm} assign nonzero weight $w$ only to causal orderings of $z_{\mathcal{K}}$, allowing feature importance to be distributed selectively. In contrast to $v$ from Def.~\ref{def:shapley_def_standard}, which typically measures the effects on the overall game outcome, we aim for measuring effects of players \emph{during} the game. We extend $v:\Sigma \rightarrow \mathbb{R}$ to consider the current game situation and propose the characteristic function $\nu:\mathcal{S} \times \Sigma \rightarrow \mathbb{R}$, for which concrete definitions are provided in Sec.~\ref{sec:charact_funcs}.

\begin{definition}
    Given an order $\sigma$ and $j \succ_{\sigma} i$ meaning $j$ succeeds $i$ in $\sigma$, we define agent specific coalitions as $\mathcal{C}_{\sigma}^i \coloneqq \{j : j \succ_{\sigma} i\}.$
    \label{def:agent_coalitions}
\end{definition}

\begin{definition}
    Let agent $i=\sigma(k)$ be the acting player. Given an order $\sigma \in \Sigma^i_+$, where $\Sigma^i_+ = \{\sigma \in \Sigma: \mathcal{C}_{\sigma}^i \neq \emptyset\}$, we define the marginal contribution of agent $i$ as the effect on members of coalition $\mathcal{C}_{\sigma}^i$ during action processing steps $k-1$ and $k$, as:
    $$\Delta \nu\left(\mathcal{C}_{\sigma}^i, s_{t, (k)}\right) \coloneqq \nu\left(\mathcal{C}_{\sigma}^i, s_{t,(k)}\right) - \nu\left(\mathcal{C}_{\sigma}^i, s_{t,(k-1)}\right).$$
    \label{def:icv_marginal}
\end{definition}

In Def.~\ref{def:icv_marginal}, the inclusion of agent $i$ is not explicit as in $\Delta v(\cdot)$ from Def.~\ref{def:shapley_def_standard}, i.e., $\mathcal{C}_{\sigma}^i \cup \{i\}$, but occurs through \emph{executing} its action $a_t^i$ and transitioning to $s_{t,(k)} \sim P^i(\cdot \mid s_{t,(k-1)}, \text{do}(A_t^i = a_t^i))$ according to Def.~\ref{def:inter_state}. This serves two purposes: 
\begin{enumerate*}[label=\textit{\roman*)}]
    \item preserving natural game execution while measuring an agent’s effect within the causal chain of $\mathcal{G}_s$,
    \item capturing agent $i$'s \emph{intention} on succeeding agents in $\mathcal{C}_{\sigma}^i$ whose actions have \emph{not yet} been processed to not attribute their impact.
\end{enumerate*}

This is visualized in Fig.~\ref{fig:Shapley_effect_inter_states}, where the affected states $\{S_{t,(\cdot)}^j\}_{j \in \mathcal{C}_{\sigma}^i}$ used for computing $\Delta \nu\left(\mathcal{C}_{\sigma}^i, s_{t,(k)}\right)$ match the processed action $a_t^i$. Accordingly, the only ordering without causal effect is the empty coalition $\mathcal{C}_{\sigma}^i= \emptyset$, which is explicitly excluded in Def.~\ref{def:icv_marginal} by $\Sigma_+^i$. To quantify an agent's average effect during the game, the weight for each $\sigma \in \Sigma_+^i$ is given as $w \coloneqq \mathrm{Pr}(\sigma) = \frac{1}{\lvert \Sigma^i_+ \rvert} = \frac{1}{n!-(n-1)!}$.

\begin{definition}
    We define the \emph{Intended Cooperation Value}~(ICV) to assign credit to agent $i$ based on its average action effect on the members of coalition $\mathcal{C}_{\sigma}^i$ during the time horizon $T$ as:
        $$\Phi_i\left(\nu\right) \coloneqq \frac{1}{T} \sum_{t=0}^{T-1} \sum_{\sigma \in \Sigma_+^i} w \cdot \Delta \nu\left(\mathcal{C}_{\sigma}^i, s_{t, (k)}\right).$$
    \label{def:icv_avg}
\end{definition}

For standard SV computation, the time complexity is $O(2^n)$. Our method increases this by a factor of $n$, since each ordering requires $n$ computations of marginal contributions. Explainability methods like ours are usually applied in simulation to gain insights prior to real-world deployment, where higher computational costs are acceptable. Nevertheless, in practice it remains computationally infeasible to consider all coalitions for large $n$~\citep{shapley_complexity}. As in previous works~\citep{shap_train_marl1}, we use Monte Carlo sampling for orders $\sigma_t$ to get an unbiased estimator $\hat{\Phi}$ for each step $t$ and by assuming $T \gg n$, as follows:
\begin{equation}
    \hat{\Phi}_i\left(\nu\right) \coloneqq \frac{1}{T} \sum_{t=0}^{T-1} \cdot \Delta \nu\left(\mathcal{C}_{\sigma_t}^i, s_{t, (k)}\right), \quad \sigma_t \sim \mathcal{U}_{\Sigma_+^i}.
    \label{eq:icv_avg_mc}
\end{equation}

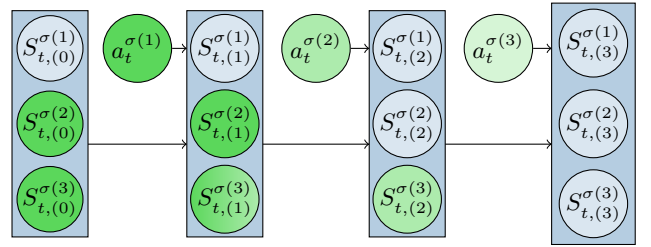
\begin{figure}[htb!]
    \centering
    \begin{tikzpicture}[scale=1]
        \draw[fill=SteelBlue!40] (0,0) rectangle (1,3);
        \node[draw,circle,minimum size=0.8cm,inner sep=0pt, fill=SteelBlue!20] (z11) at (0.5,2.5) {$S^{\sigma(1)}_{t,(0)}$};
        \node[draw,circle,minimum size=0.8cm,inner sep=0pt, fill=LimeGreen!80] (z12) at (0.5,1.5) {$S^{\sigma(2)}_{t,(0)}$};
        \node[draw,circle,minimum size=0.8cm,inner sep=0pt, fill=LimeGreen!80] (z13) at (0.5,0.5) {$S^{\sigma(3)}_{t,(0)}$};

        \node[draw,circle,minimum size=0.9cm,inner sep=0pt, fill=LimeGreen!80] (a1) at (1.65,2.5) {$a^{\sigma(1)}_t$};


        \draw[->] (1,1.23) -- (2.3,1.23);
        \draw[->] (a1) -- (2.3,2.5);

        \draw[fill=SteelBlue!40] (2.3,0) rectangle (3.3,3);
        
        \node[draw,circle,minimum size=0.9cm,inner sep=0pt, fill=SteelBlue!20] (z21) at (2.8,2.5) {$S^{\sigma(1)}_{t,(1)}$};
        \node[draw,circle,minimum size=0.9cm,inner sep=0pt, fill=LimeGreen!80] (z22) at (2.8,1.5) {$S^{\sigma(2)}_{t,(1)}$};
        \node[draw,circle,minimum size=0.9cm,inner sep=0pt, shading=axis, left color=LimeGreen!80, right color=LimeGreen!40, shading angle=90] (z23) at (2.8,0.5) {$S^{\sigma(3)}_{t,(1)}$};

        \node[draw,circle,minimum size=0.9cm,inner sep=0pt, fill=LimeGreen!40] (a2) at (4.0,2.5) {$a_t^{\sigma(2)}$};

        \draw[->] (3.3,1.23) -- (4.7,1.23);
        \draw[->] (a2) -- (4.7,2.5);

        \draw[fill=SteelBlue!40] (4.7,0) rectangle (5.7,3);

        \node[draw,circle,minimum size=0.9cm,inner sep=0pt, fill=SteelBlue!20] (z31) at (5.2,2.5) {$S^{\sigma(1)}_{t,(2)}$};
        \node[draw,circle,minimum size=0.9cm,inner sep=0pt, fill=SteelBlue!20] (z32) at (5.2,1.5) {$S^{\sigma(2)}_{t,(2)}$};
        \node[draw,circle,minimum size=0.9cm,inner sep=0pt, fill=LimeGreen!40] (z33) at (5.2,0.5) {$S^{\sigma(3)}_{t,(2)}$};

        \node[draw,circle,minimum size=0.9cm,inner sep=0pt, fill=LimeGreen!20] (a3) at (6.4,2.5) {$a_t^{\sigma(3)}$};

        \draw[->] (5.7,1.23) -- (7.1,1.23);
        \draw[->] (a3) -- (7.1,2.5);

        \draw[fill=SteelBlue!40] (7.1,-0.1) rectangle (8.2,3.1);

        \node[draw,circle,minimum size=0.9cm,inner sep=0pt, fill=SteelBlue!20] (z31) at (7.65,2.57) {$S^{\sigma(1)}_{t,(3)}$};
        \node[draw,circle,minimum size=0.9cm,inner sep=0pt, fill=SteelBlue!20] (z32) at (7.65,1.5) {$S^{\sigma(2)}_{t,(3)}$};
        \node[draw,circle,minimum size=0.9cm,inner sep=0pt, fill=SteelBlue!20] (z33) at (7.65,0.44) {$S^{\sigma(3)}_{t,(3)}$};
    \end{tikzpicture}
    \vspace{0.2cm}
    \caption{Intermediate states $S_{t,(\cdot)}^{\sigma(\cdot)}$ used for computing the marginal contributions $\Delta \nu$ match the color of the processed action $a_t^{\sigma(k)}$. Notice that $S_{t,(1)}^{\sigma(3)}$ is used for the effects of both $a_t^{\sigma(1)}$ and $a_t^{\sigma(2)}$. Arcs entering $a_t^{\sigma(k)}$ are explicitly ignored due to $\mathrm{do}$-intervention.}
    \vspace{0.5cm}
    \label{fig:Shapley_effect_inter_states}
\end{figure}

\subsection{Characteristic Functions}\label{sec:charact_funcs}

We present various definitions for $\nu$ based on Def.~\ref{def:input_features} and~\ref{def:shapley_def_standard}.

\paragraph{Value-based.} \citet{shapley_rl1} derived the correct application of SVs in single-agent RL by introducing the following characteristic functions, with $p^{\pi}(z'\mid z_{\mathcal{K}})$ as the state occupancy distribution:
\begin{align}
    &v_V(\mathcal{K}) = \sum_{z' \in \mathcal{Z}} p^{\pi}(z'\mid z_{\mathcal{K}}) V(z'),\\
    &v_{\pi}(\mathcal{K}) = \sum_{z' \in \mathcal{Z}} p^{\pi}(z'\mid z_{\mathcal{K}}) \pi(a\mid z').
\end{align}
They showed that SVs with $v_{V}(\mathcal{K})$ capture the contribution of state features to the value function (Eq.~\eqref{eq:value_function}) as a \emph{static predictor}, and with $v_{\pi}(\mathcal{K})$, their influence on the probability of selecting action $a$. In contrast, $\nu(\mathcal{C}_{\sigma}^i, s_{t,(k)})$ does \emph{not} evaluate subsets of features, but full realizations of intermediate states $s_{t,(k)}$, where feature contributions emerge from the effect of an agent’s actions on the state.

\begin{definition}
    The characteristic value-based function is defined as:
        $$\nu_{v}(\mathcal{C}_{\sigma}^i, s_{t,(k)}) \coloneqq \sum_{j\in \mathcal{C}_{\sigma}^i} V^j(s_{t,(k)}),$$
    with $V^j(\cdot)$ being defined in Eq.~\eqref{eq:value_function}. This reduces to $\nu_{v}(\mathcal{C}_{\sigma}^i, s_{t,(k)}) = V(s_{t,(k)})$ in the case of centralized critics.
    \label{def:v_v}
\end{definition}

\paragraph{Entropy-based.} As noted in Sec.\ref{sec:intro}, access to many future states is instrumentally valuable. The size of future accessible states can be quantified by $\lvert\mathcal{S}_+\rvert$, where $\mathcal{S}_+ = \{s_{t+1} \in \mathcal{S} : P(s_{t+1} \mid s_t) > 0\}$. Based on this, \citet{altruism_paper} empirically validated the use of $H(A_t \mid s_t)$ as a proxy for $\lvert \mathcal{S}_+ \rvert$ to capture an agent’s future choices. Similarly, \citet{shapley_uncertainty} introduced information-theoretic SVs via $v = H(Y \mid z_{\mathcal{K}})$ to measure predictive uncertainty in ML, with $Y$ being the output. Building on both, we argue that in some multi-agent settings, it may be more valuable for an agent to be \emph{certain} about what to do, i.e., committed to a specific strategy. 

\begin{definition}
    Since $H(A_{t,(k+1)}^i \mid s_{t,(k)}) \leq \log|\mathcal{A}^i|$~\cite{entropy_log} quantifies the spread of $\pi^i(\cdot \mid s_{t,(k)})$, we define the complement to this bound as a measure of decision certainty, denoted as \emph{peakedness}: 
    $$\mathcal{H}(A_{t,(k+1)}^j) \coloneqq \log\lvert \mathcal{A}^j \rvert - H(A_{t,(k+1)}^j \mid s_{t,(k)}).$$
    The characteristic peak-based function is then given as:
        $$\nu_p(\mathcal{C}_{\sigma}^i, s_{t,(k)}) \coloneqq \sum_{j\in \mathcal{C}_{\sigma}^i} \mathcal{H}\left( A_{t,(k+1)}^j \right).$$
    \label{def:v_p}
\end{definition}

\paragraph{Consensus-based.} While the policy distribution may have the same shape even if the agent chooses a different action, $H(A_t)$ does not capture such changes. To address this, and to assess whether agents act in line with expectations to others, we use the \emph{Jensen–Shannon Divergence}~(JSD), which has been used to encourage diversity among agents~\citep{intr_rew_diversity}. For two PMFs $p$ and $q$ with $u = \frac{1}{2}(p + q)$, the JSD is a symmetrized version of $D_{KL}$:
\begin{equation}
    \bar{\mathcal{J}} (p \mid \mid q) \coloneqq \mathrm{JSD}(p \mid \mid q) =  \frac{1}{2} D_{KL}(p \mid \mid u) + \frac{1}{2} D_{KL}(q \mid \mid u).
    \label{eq:jsd}
\end{equation}

While in the ordinary MG framework (Def.~\ref{def:mg}) each agent $i$ observes the global state $s_t$, in practice agents either learn which components of $S_t$ pertain to them, or the environment assigns agent $i$'s components to a fixed (typically first) position. We denote by $s_t^i$ the value of agent $i$'s component in the global state, and by $s_t^{(i)}$ the agent's \emph{individualized} observation. Based on this, we define a consensus metric to include the cases of:
\begin{enumerate*}[label=\textit{\roman*)}]
\item \emph{others-consensus}, to infer if agent $i$ would act similarly in the state of others and,
\item \emph{self-consensus}, to infer if others would act similarly in the state agent $i$.
\end{enumerate*}

\begin{definition} Using base-2 $\log$, $\bar{\mathcal{J}}(p \| q)$ is bounded in $[0, 1]$, where $0$ indicates identical distributions and $1$ maximal divergence. To instead express similarity, we define $\mathcal{J}(p \| q) \coloneqq 1 - \bar{\mathcal{J}}(p \| q)$. The characteristic consensus-based function is thus given by:
    \begin{align}
        \nu_c(\mathcal{C}_{\sigma}^i, s_{t,(k)}) \coloneqq & \sum_{j\in \mathcal{C}_{\sigma}^i} \bigg[ \overbrace{\mathcal{J} \left(\pi^i\left(\cdot \mid s^{(j)}_{t,(k)}\right) \mid\mid \pi^j\left(\cdot \mid s^{(j)}_{t,(k)}\right)\right)}^{other} \nonumber \\
        & + \underbrace{\mathcal{J} \left(\pi^i\left(\cdot \mid s^{(i)}_{t,(k)}\right) \mid\mid \pi^j\left(\cdot \mid s^{(i)}_{t,(k)}\right)\right)}_{self} \bigg]. \nonumber
    \end{align}
    If only one of the consensus terms is used, we denote \emph{others-consensus} by $\nu_{co}$ and \emph{self-consensus} by $\nu_{cs}$. For dissimilarity, we use $\bar{\mathcal{J}}$ in place of $\mathcal{J}$ in $\nu_c$ and term it $\nu_{d}$, $\nu_{do}$ and $\nu_{ds}$ respectively.
    \label{def:v_c}
\end{definition}

\subsection{Instrumental Empowerment}\label{sec:instr_empower}

For each definition of $\nu$, we link the marginal contribution from Def.~\ref{def:icv_marginal} to important and well-known metrics in MARL. In the following, we use the shorthand $\Delta \nu$ to denote $\Delta \nu (\mathcal{C}_{\sigma}^i, s_{t,(k)})$.

\begin{definition}
The advantage of agent $j$ is given by
$\boldsymbol{\Lambda}^j(s_{t}, \boldsymbol{a}_t) \coloneqq  \mathbb{E}_{s_{t+1}\sim P(\cdot|s_t,\boldsymbol{a}_t)} [R^j(s_t, \boldsymbol{a}_t,s_{t+1}) + \gamma V^j(s_{t+1})] - V^j(s_{t})$. We define the sampled advantage as: 
$$\boldsymbol{\hat{\Lambda}}^j(s_{t}, \boldsymbol{a}_t, s_{t+1}) \coloneqq R^j(s_t, \boldsymbol{a}_t,s_{t+1}) + \gamma V^j(s_{t+1}) - V^j(s_{t}).$$
\label{def:advantage_func}
\end{definition}

\begin{proposition} 
Given $V^j(s)$ as defined by Eq.~\eqref{eq:value_function} that is learned under $\mathcal{G}_m$ dynamics, and employing it as a \emph{static predictor} within $\mathcal{G}_s$, then for each transition of intermediate states $s_{t,(k-1)} \rightarrow s_{t,(k)}$, the value-based marginal contribution $\Delta \nu_v$ measures the undiscounted sampled advantage for players $j \in \mathcal{C}_{\sigma}^i$ attributable to agent $i$:
$$\Delta \nu_{v} = \sum_{j \in \mathcal{C}_{\sigma}^i} \hat{\Lambda}^j(s_{t,(k-1)}, a_t^i, s_{t, (k)}).$$
\end{proposition}

\begin{proof}
See Appendix~\ref{app:adv_proof}.
\end{proof}

\begin{definition} 
    The ordinary empowerment function~\citep{empowerment_def} quantifies an agent's maximal potential causal influence on future \emph{states}. We define the \emph{instrumental empowerment} of the agent $i = \sigma(k)$ as the potential influence on future \emph{actions} of agents $j \in \mathcal{C}^i_{\sigma}$ as follows:
    $$\mathcal{E}^j(s_{t,(k-1)}, i) \coloneqq \max_{\pi^i \in (\mathcal{P} (\mathcal{A}^i))^{\mathcal{S}}} \, \mathrm{I}(A_{t,(k)}^j ; a_t^i \mid s_{t,(k-1)}). $$
    \label{def:instr_empowerment}
\end{definition}

\begin{proposition}  
For each intermediate state $s_{t,(k)} \in \mathcal{S}$, the entropy-based marginal contribution $\Delta \nu_{p}$ measures the mutual information $\mathrm{I}(A_{t,(k)}^j ; a_t^i \mid s_{t,(k-1)})$ as the certainty increase in decision-making of players $j \in \mathcal{C}_{\sigma}^i$ attributable to agent $i$'s action $a_t^i$:
    $$\Delta \nu_{p} = \sum_{j\in \mathcal{C}_{\sigma}^i} \mathrm{I}(A_{t,(k)}^j ; a_t^i \mid s_{t,(k-1)}).$$
If moreover the acting policy $\pi^i$ maximizes each conditional mutual information, then $\Delta \nu_{p}$ induces the \emph{instrumental empowerment:}
$$\Delta \nu_{p} =\sum_{j\in \mathcal{C}_{\sigma}^i} \mathcal{E}^j(s_{t,(k-1)}, i).$$
\end{proposition}

\begin{proof}
See Appendix~\ref{app:empower_proof}.
\end{proof}

\begin{definition}
    An agent $i$'s action $a^{i*}$ is a \emph{best response} to actions of co-players $-\boldsymbol{i}$ if it maximizes its payoff $u^i$ given the strategies chosen by the others, formally defined as $a^{i*}=\mathrm{argmax}_{a^i\in \mathcal{A}^i}u^i(a^i, a^{-\boldsymbol{i}})$.
    \label{def:best_resp}
\end{definition}

\begin{proposition}
    At state $s_{t,(k)}$, the action $a_t^i$ increases consensus to co-players $j \in \mathcal{C}_{\sigma}^i$ if $\Delta \nu_j \geq 0$. If furthermore $\Delta \nu_c (\cdot \mid a_t^i) = \max_{a \in \mathcal{A}^i} \Delta \nu_c(\cdot \mid a)$, then $a_t^i$ is a \emph{best response} at $s_{t,(k)}$ in terms of alignment with co-player strategies. 
\end{proposition}

\begin{proof}
    Follows directly from Def.~\ref{def:best_resp}, as maximizing $\nu_c$ is equivalent to maximizing $\Delta \nu_c$.
\end{proof}

\section{Experiments}\label{sec:experiments}
To assess whether the intuition from Fig.~\ref{fig:toy_env_concept} and our ICV method generalize to more complex settings, we conduct empirical validation in three environments.\footnote{Python code available at \href{https://github.com/IDSIA-papers/ICV_Shapley_MARL}{github.com/icv-marl}.} Each illustrates distinct applications of ICVs \emph{online} and \emph{offline} for extracting behavioral insights and identifying credit assignment in the absence of rewards. We examine whether $\Phi(\nu_p)$ and $\Phi(\nu_c)$ can serve as proxies for $\Phi(\nu_v)$. Thus, our reward-free method can be seen as an extension of the reward-based approaches in~\cite{shapley_rl1, shapley_marl_importance, shapley_rl2}, particularly since the validity of using the value-based function $\Phi(\nu_v)$ as a feature contribution predictor has been proven in~\citep{shapley_rl1}. We compute ICVs over $M$ episodes and report \emph{relative} ICV values by $i)$ normalizing all characteristic functions $\nu$, and $ii)$ dividing each $\Phi_i(\cdot)$ by $\kappa = \max_{i,j} \Phi_i(\nu_j)$ to scale results close to $1$, as some contributions may cancel out and yield low contributions on average. We consider pre-trained models, with training similar to~\cite{marl_benchmarks}, and assume convergence of the learning process but note that our method is applicable at any point during training.

\subsection{Cooperative Game}\label{sec:lbf_exp}
We consider the fully cooperative task in \emph{Level-based Foraging}~(LBF), illustrated in Fig.~\ref{fig:lbf_env}, where the numbering of agents and apples serve as \emph{instance identifiers} (rather than level indicators). All agents must surround the same apple on separate cells and then perform a simultaneous \emph{load} action to receive a reward.

Each agent uses an individual centralized critic, and our ICV method is applied \emph{online}. We first validate the correspondence between value $V^i$ and peak $\mathcal{H}^i$ increase in Fig.~\ref{fig:lbf_env_policy_both}. In the depicted initial state (left), agent 1 moves to the right (green arrow), thereby increasing both its own and its teammates’ certainty in action selection, while also increasing the value for all, which is shown in the right plot. This effect arises because agent 1 positions itself below the apple, therefore clarifying to others where they should move to.

\begin{figure}[htb!]
    \centering
    \begin{subfigure}{0.37\columnwidth}
        \centering
        \includegraphics[width=\linewidth]{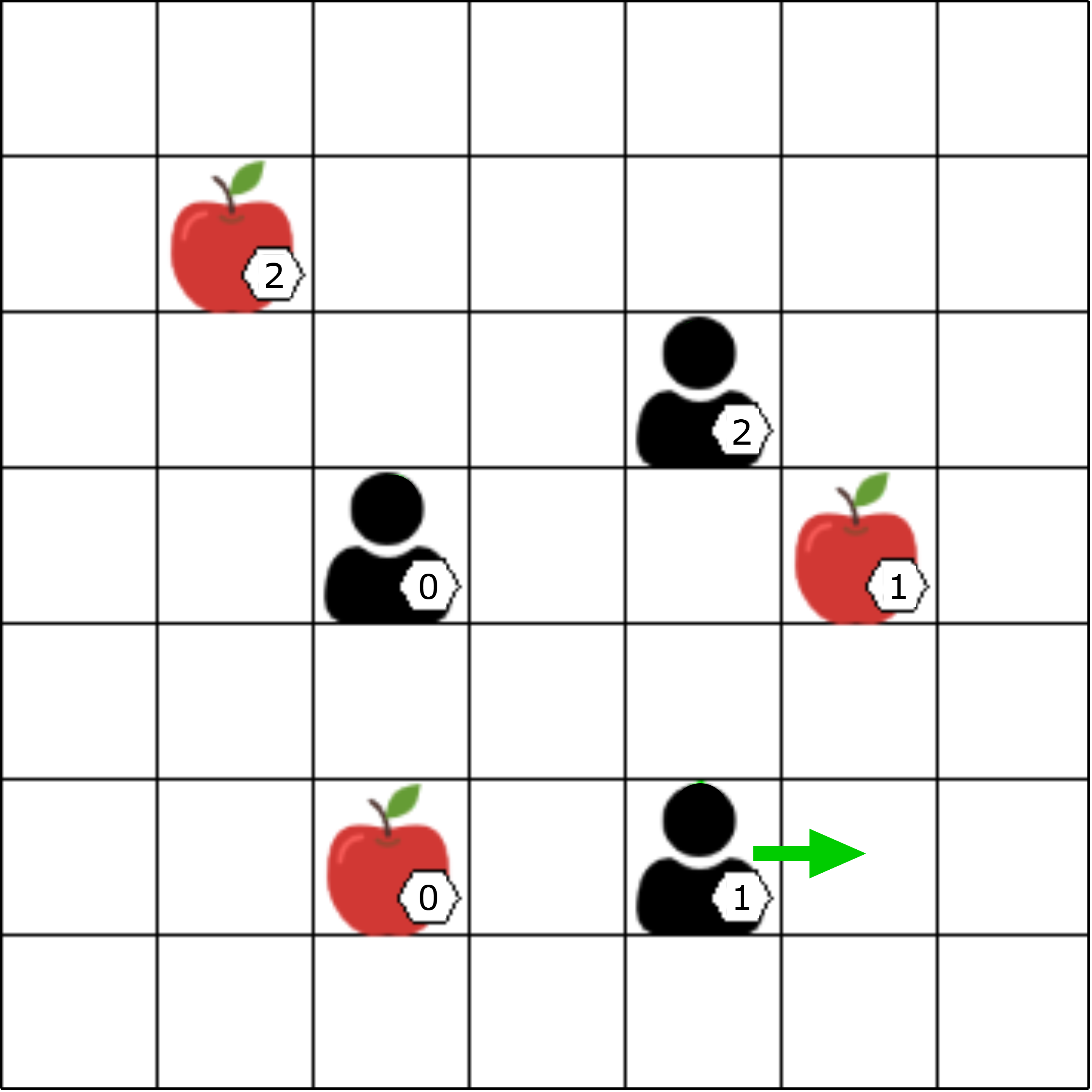}
        \caption{LBF environment.}
        \label{fig:lbf_env}
    \end{subfigure}
    \hspace{0.1cm}
    \begin{subfigure}{0.6\columnwidth}
        \centering
        \vbox{
        \begin{tikzpicture}
            \begin{axis}[
                width=\columnwidth, height=2cm,
                hide axis,
                xmin=0, xmax=0, ymin=0, ymax=0,
                legend columns=6,
                legend style={at={(0.6,0.0)},anchor=center},
            ]
            \addlegendimage{only marks, color=NavyBlue}
            \addlegendentry{{\scriptsize up} \hspace{0.15cm}}
            \addlegendimage{only marks, color=RoyalBlue}
            \addlegendentry{{\scriptsize down} \hspace{0.15cm}}
            \addlegendimage{only marks, color=SteelBlue!80}
            \addlegendentry{{\scriptsize left} \hspace{0.15cm}}
            \addlegendimage{only marks, color=DimGray}
            \addlegendentry{{\scriptsize right} \hspace{0.15cm}}
            \addlegendimage{only marks, color=LightSlateGray!70}
            \addlegendentry{{\scriptsize load} \hspace{0.15cm}}
            \addlegendimage{only marks, color=BrickRed}
            \addlegendentry{{\scriptsize value}}
            \end{axis}
        \end{tikzpicture}
        }
        \vspace{-0.05cm}
        \vbox{
        \begin{tikzpicture}
            \begin{axis}[
                width=\columnwidth+0.5cm,
                height=2.3cm,
                enlarge x limits=0.25,
                xtick style={/pgfplots/major tick length=0pt},
                xtick={1, 2, 3},
                xticklabels={$\pi^0(s_t)$, $\pi^1(s_t)$, $\pi^2(s_t)$},
                ybar,
                bar width=3.5pt,
                ymin=0, ymax=1,
                ytick distance=1,
                ticklabel style={font=\footnotesize},
                xticklabel style={yshift=4pt},
            ] 
            \draw[solid, DarkSlateGray!50] (axis cs:1.5,0) -- (axis cs:1.5,1);
            \draw[solid, DarkSlateGray!50] (axis cs:2.5,0) -- (axis cs:2.5,1);

            \draw[solid, BrickRed] (axis cs:0,0.2) -- (axis cs:1.5,0.2);
            \draw[solid, BrickRed] (axis cs:1.5,0.3) -- (axis cs:2.5,0.3);
            \draw[solid, BrickRed] (axis cs:2.5,0.25) -- (axis cs:3.5,0.25);
            
            \addplot[fill=NavyBlue , draw= NavyBlue]
                coordinates {(1,0.35) (2,0.4) (3,0.1)};
            \addplot[fill=RoyalBlue , draw= RoyalBlue]
                coordinates {(1,0.15) (2,0.05) (3,0.4)};
            \addplot[fill=SteelBlue!80 , draw= SteelBlue!80]
                coordinates {(1,0.05) (2,0.05) (3,0.0)};
            \addplot[fill=DimGray , draw=DimGray]
                coordinates {(1, 0.45) (2, 0.5) (3, 0.5)};
            \addplot[fill=LightSlateGray!70 , draw=LightSlateGray!70]
                coordinates {(1, 0.0) (2, 0.0) (3, 0.0)};
            
            \end{axis}
        \end{tikzpicture}
        }
        \vspace{-0.1cm}
        \vbox{
        \begin{tikzpicture}
            \begin{axis}[
                width=\columnwidth+0.5cm,
                height=2.3cm,
                enlarge x limits=0.25,
                xtick style={/pgfplots/major tick length=0pt},
                xtick={1, 2, 3},
                xticklabels={$\pi^0(s_t)$, $\pi^1(s_t)$, $\pi^2(s_t)$},
                ybar,
                bar width=3.5pt,
                ymin=0, ymax=1,
                ytick distance=1,
                ticklabel style={font=\footnotesize},
                xticklabel style={yshift=4pt},
            ] 
            \draw[solid, DarkSlateGray!50] (axis cs:1.5,0) -- (axis cs:1.5,1);
            \draw[solid, DarkSlateGray!50] (axis cs:2.5,0) -- (axis cs:2.5,1);

            \draw[solid, BrickRed] (axis cs:0,0.27) -- (axis cs:1.5,0.27);
            \draw[solid, BrickRed] (axis cs:1.5,0.39) -- (axis cs:2.5,0.39);
            \draw[solid, BrickRed] (axis cs:2.5,0.33) -- (axis cs:3.5,0.33);

            \draw[solid, BrickRed!30] (axis cs:0,0.2) -- (axis cs:1.5,0.2);
            \draw[solid, BrickRed!30] (axis cs:1.5,0.3) -- (axis cs:2.5,0.3);
            \draw[solid, BrickRed!30] (axis cs:2.5,0.25) -- (axis cs:3.5,0.25);
            
            \addplot[fill=NavyBlue , draw= NavyBlue]
                coordinates {(1,0.1) (2,0.9) (3,0.0)};
            \addplot[fill=RoyalBlue , draw= RoyalBlue]
                coordinates {(1,0.05) (2,0.0) (3,0.35)};
            \addplot[fill=SteelBlue!80 , draw= SteelBlue!80]
                coordinates {(1,0.0) (2,0.0) (3,0.0)};
            \addplot[fill=DimGray , draw=DimGray]
                coordinates {(1, 0.85) (2, 0.05) (3, 0.6)};
            \addplot[fill=LightSlateGray!70 , draw=LightSlateGray!70]
                coordinates {(1, 0.0) (2, 0.05) (3, 0.05)};
            
            \end{axis}
        \end{tikzpicture}
        }
        \vspace{-0.1cm}
        \caption{Policy distributions and actions.}
        \label{fig:lbf_policy_distributions}
    \end{subfigure}
\vspace{0.1cm}
\caption{The effect of agent 1's action (moving right) on the policies and values before (top) and after (bottom) acting.}
\label{fig:lbf_env_policy_both}
\vspace{0.7cm}
\end{figure}
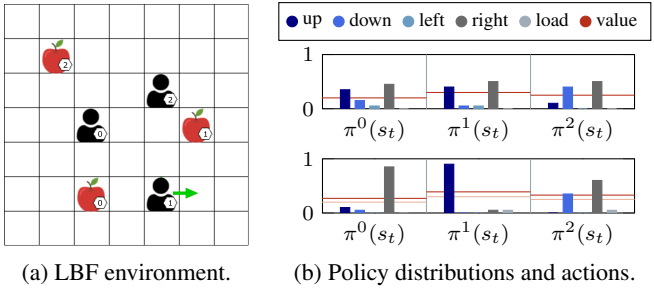

Contrary to the previous, we verify the correspondence between increase in value $V^i(\cdot)$ and uncertainty in decision-making, measured by $H(A_t^i)$ for each agent. Besides the entropy, we count the number of actions $a_t^i$ yielding $\Lambda^i(s_t,\boldsymbol{a}_t^{\boldsymbol{-i}}, a_t^i) \geq 0$ as per Def.~\ref{def:advantage_func}, by keeping the actions $\boldsymbol{a^{-i}}$ fixed and term this as agent $i$'s \emph{choice} $C^i$. We plot the normalized values of $V^i$, entropy $H^i=H(A_t^i)$ and choices $C^i$ in Fig.~\ref{fig:lbf_history_choice}. The plots per agent $i$ indicate a reasonable correspondence between increases in value and decrease of entropy and number of valuable choices, thereby further supporting our intuition.

\usetikzlibrary{fillbetween}
\begin{figure}[htb!]
\centering
\begin{subfigure}{0.3\columnwidth}
    \centering
    \vbox{
    \begin{tikzpicture}
    \hspace{3.1cm}
        \begin{axis}[
            width=\columnwidth, height=2cm,
            hide axis,
            xmin=0, xmax=0, ymin=0, ymax=0,
            legend columns=3,
            legend style={at={(0,0.0)},anchor=center}
        ]
        \addlegendimage{only marks, color=NavyBlue}
        \addlegendentry{$V^i$ \hspace{0.2cm}}
        \addlegendimage{only marks, color=RoyalBlue}
        \addlegendentry{$H^i$ \hspace{0.2cm}}
        \addlegendimage{only marks, color=DimGray}
        \addlegendentry{$C^i$}
        \end{axis}
    \end{tikzpicture}
    }
    \vspace{-0.05cm}
    \vbox{
    \begin{tikzpicture}
      \begin{axis}[
            width=\columnwidth+1.2cm, height=3cm,
            xmin=-1, xmax=15,
            ymin=0.0, ymax=1,
            xtick distance=5,
            ytick distance=1,
            enlargelimits=false,
            ticklabel style={font=\footnotesize},
            every axis plot/.append style={thick},
            grid=none,
        ]
    
        \addplot[
          NavyBlue,
          very thick,
          name path=avg_v0
        ] table [x=x, y=avg_v0] {exp/lbf_history_choice_entropy.dat};
        
        \addplot[
          RoyalBlue,
          very thick,
          name path=avg_h0
        ] table [x=x, y=avg_h0] {exp/lbf_history_choice_entropy.dat};

        \addplot[
          DimGray,
          very thick,
          name path=avg_c0
        ] table [x=x, y=avg_c0] {exp/lbf_history_choice_entropy.dat};
      \end{axis}
    \end{tikzpicture}
    \caption{Agent 0.}
    \label{fig:lbf_history_choice_P0}
    }
\end{subfigure}
\hspace{0.2cm}
\begin{subfigure}{0.3\columnwidth}
    \centering
    \begin{tikzpicture}
      \begin{axis}[
            width=\columnwidth+1.2cm, height=3cm,
            xmin=-1, xmax=15,
            ymin=0.0, ymax=1,
            xtick distance=5,
            ytick distance=1,
            enlargelimits=false,
            ticklabel style={font=\footnotesize},
            every axis plot/.append style={thick},
            grid=none,
        ]
    
        \addplot[
          NavyBlue,
          very thick,
          name path=avg_v1
        ] table [x=x, y=avg_v1] {exp/lbf_history_choice_entropy.dat};
        
        \addplot[
          RoyalBlue,
          very thick,
          name path=avg_h1
        ] table [x=x, y=avg_h1] {exp/lbf_history_choice_entropy.dat};

        \addplot[
          DimGray,
          very thick,
          name path=avg_c1
        ] table [x=x, y=avg_c1] {exp/lbf_history_choice_entropy.dat};
      \end{axis}
    \end{tikzpicture}
    \caption{Agent 1.}
    \label{fig:lbf_history_choice_P1}
\end{subfigure}
\hspace{0.2cm}
\begin{subfigure}{0.3\columnwidth}
    \centering
    \begin{tikzpicture}
      \begin{axis}[
            width=\columnwidth+1.2cm, height=3cm,
            xmin=-1, xmax=15,
            ymin=0.0, ymax=1,
            xtick distance=5,
            ytick distance=1,
            enlargelimits=false,
            ticklabel style={font=\footnotesize},
            every axis plot/.append style={thick},
            grid=none,
        ]
    
        \addplot[
          NavyBlue,
          very thick,
          name path=avg_v2
        ] table [x=x, y=avg_v2] {exp/lbf_history_choice_entropy.dat};
        
        \addplot[
          RoyalBlue,
          very thick,
          name path=avg_h2
        ] table [x=x, y=avg_h2] {exp/lbf_history_choice_entropy.dat};

        \addplot[
          DimGray,
          very thick,
          name path=avg_c2
        ] table [x=x, y=avg_c2] {exp/lbf_history_choice_entropy.dat};
      \end{axis}
    \end{tikzpicture}
    \caption{Agent 2.}
    \label{fig:lbf_history_choice_P2}
\end{subfigure}
\vspace{0.5cm}
\caption{Comparison between normalized values $V^i$, entropy $H^i$ and choices $C^i$ along y-axis over time steps $t$ in x-axis of LBF agents.}
\label{fig:lbf_history_choice}
\vspace{0.7cm}
\end{figure}
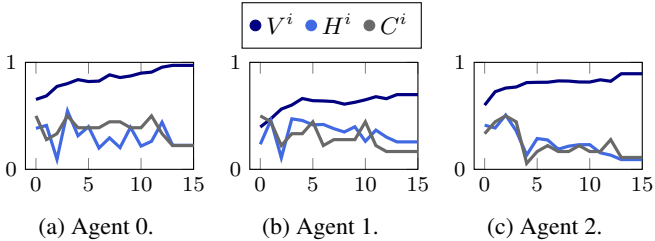

Finally, we plot the sampled ICV values $\hat{\Phi}(\cdot)$ in Fig.~\ref{fig:lbf_results}, averaged over $M=500$ episodes and normalized using $\kappa \approx 0.06$. We additionally report the \emph{changes} of choices $\delta (C^i)$ per agent. Consistent with Fig.~\ref{fig:lbf_history_choice}, the results indicate that agents contribute valuably to their co-players and receive credit by increasing their decision certainty, as also reflected in the reduced number of choices $\delta (C^i)$. We further observe increased consensus ($\hat{\Phi}(\nu_{cs})$) on the acting agents A0 and A1, suggesting that they acted as anticipated by others, whereas A2 followed different strategies despite having identical role. The minor negative effects on $\hat{\Phi}(\nu_{co})$ imply that non-acting players' preferences were not fully satisfied, which mainly affects A0 and A2 in this case.

\begin{figure}[htb!]
    \centering
    \hspace{-0.8cm}
    \begin{subfigure}{0.8\columnwidth}
        \centering
        \begin{tikzpicture}
            \begin{axis}[
                width=\columnwidth+1cm,
                height=3.5cm,
                enlarge x limits=0.25,
                xtick style={/pgfplots/major tick length=0pt},
                xtick={1, 2, 3},
                xticklabels={A0, A1, A2},
                ybar,
                bar width=7pt,
                ymin=-1, ymax=1,
                ytick distance=0.5,
                ticklabel style={font=\footnotesize},
                xticklabel style={yshift=4pt},
            ]
            \draw[solid, DarkSlateGray!50] (axis cs:1.5,-1) -- (axis cs:1.5,1);
            \draw[solid, DarkSlateGray!50] (axis cs:2.5,-1) -- (axis cs:2.5,1);
            
            \addplot[fill=NavyBlue , draw= NavyBlue]
                coordinates {(1,0.48) (2,0.55) (3,0.6)};
            \addplot[fill= RoyalBlue, draw= RoyalBlue]
                coordinates {(1,0.95) (2,0.65) (3,0.55)};
            \addplot[fill=SteelBlue!80 , draw=SteelBlue!80]
                coordinates {(1,-0.87) (2,-0.75) (3, -0.7)};
            \addplot[fill=DimGray, draw=DimGray]
                coordinates {(1, 0.4) (2, 0.3) (3, -0.2)};
            \addplot[fill=LightSlateGray!70, draw=LightSlateGray!70]
                coordinates {(1, -0.3) (2, -0.1) (3, -0.26)};
            \end{axis}
        \end{tikzpicture}
    \end{subfigure}
    \hspace{0.2cm}
    \begin{subfigure}{0.1\columnwidth}
        \centering
            \begin{tikzpicture}
                \begin{axis}[
                    width=2cm, height=2cm,
                    hide axis,
                    xmin=0, xmax=0, ymin=0, ymax=0,
                    legend columns=1,
                    legend style={at={(0.0,0.0)},anchor=center}
                ] 
                \addlegendimage{only marks, color=NavyBlue}
                \addlegendentry{\footnotesize{$\hat{\Phi}(\nu_v)$}}
                \addlegendimage{only marks, color=RoyalBlue}
                \addlegendentry{\footnotesize{$\hat{\Phi}(\nu_p)$}}
                \addlegendimage{only marks, color=SteelBlue!80}
                \addlegendentry{\footnotesize{$\delta (C^i)$}}
                \addlegendimage{only marks, color= DimGray}
                \addlegendentry{\footnotesize{$\hat{\Phi}(\nu_{cs})$}}
                \addlegendimage{only marks, color=LightSlateGray!70}
                \addlegendentry{\footnotesize{$\hat{\Phi}(\nu_{co})$}}
                \end{axis}
            \end{tikzpicture}
            \vspace{0.03cm}
      \end{subfigure}
    \vspace{0.2cm}
    \caption{ICV $\hat{\Phi}(\cdot)$ and choice changes $\delta (C^i)$ on LBF (y-axis) showing the effect of the respective agent on co-players (x-axis).}
    \label{fig:lbf_results}
    \vspace{0.8cm}
\end{figure}
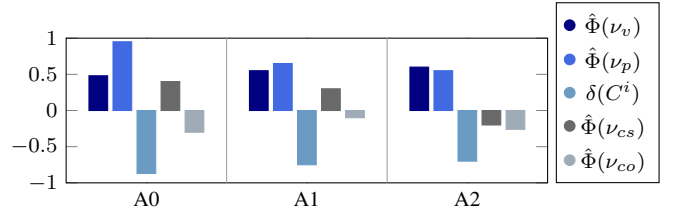

\subsection{Mixed-Motive Game}\label{sec:mpe_exp}
Here, we use the mixed-motive Tag game (Fig.~\ref{fig:mpe_env}) of \emph{Multi-Particle Environment}~(MPE), where predators (red) aim to catch a prey (green). Agents have individual critics, and ICVs are computed \emph{offline}. In Fig.~\ref{fig:mpe_history}, we analyze agent 1 (prey) and again verify the correspondence between its value $V^1$ and action determinism $\mathcal{H}^1$ over a single episode. The results indicate a reasonable alignment between value and action certainty. We further assess agent 1’s dissimilarity to its opponents (predators) by setting $\mathcal{C}_{\sigma}^i$ accordingly and measuring $\nu_{d}$, denoted as $\bar{\mathcal{J}}^1 $. Agent 1 shows high dissimilarity throughout the episode which is consistent with the intuition that a prey should act contrary to predators, ideally avoiding movement toward them.

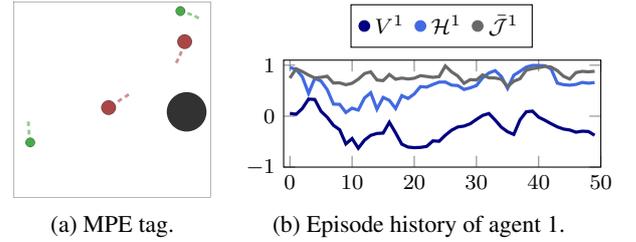
\begin{figure}[htb!]
\centering
    \begin{subfigure}{0.3\columnwidth}
        \centering
        \includegraphics[width=\linewidth]{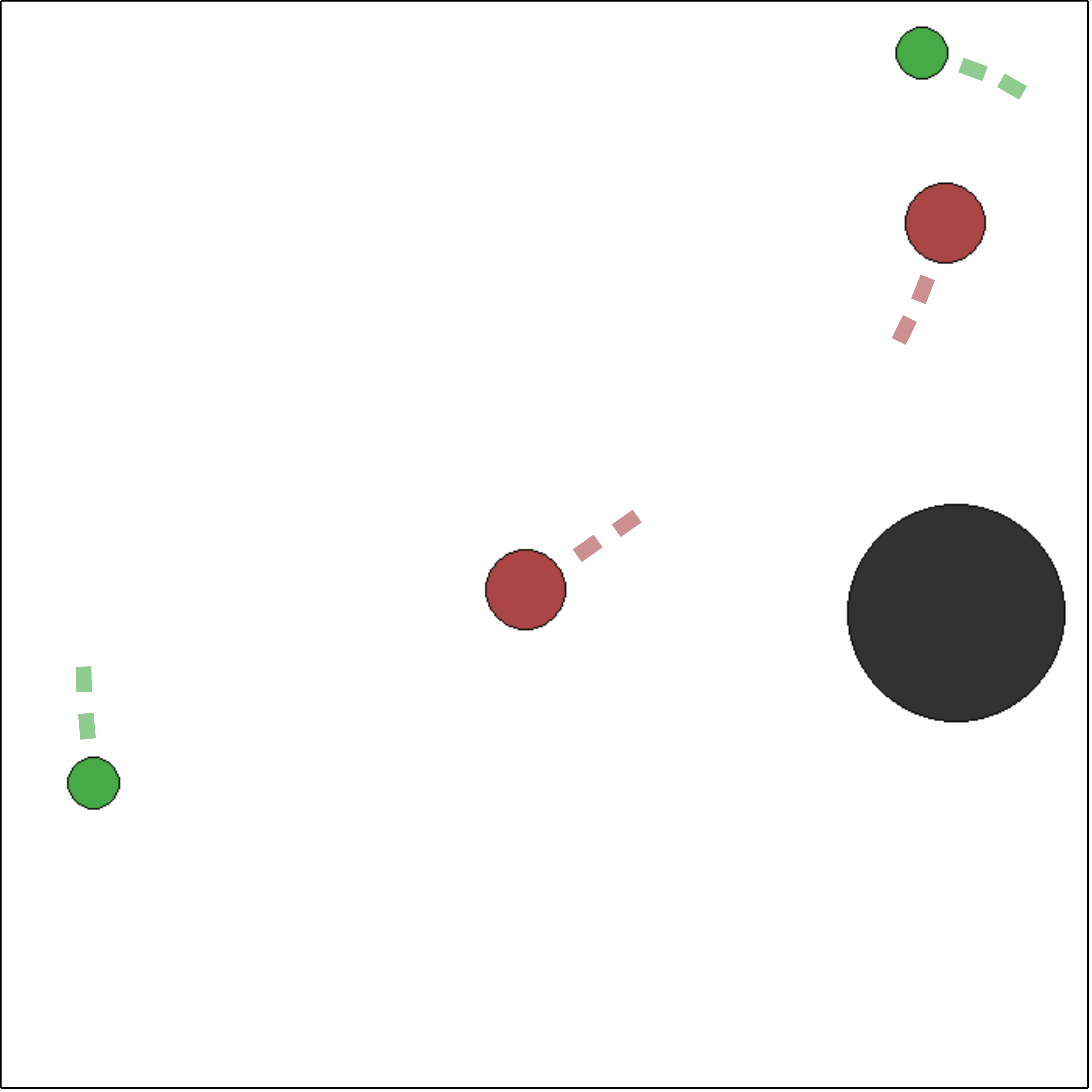}
        \caption{MPE tag.}
        \label{fig:mpe_env}
    \end{subfigure}
    \hspace{0.2cm}
    \begin{subfigure}{0.55\columnwidth}
    \centering
    \begin{tikzpicture}
    \hspace{0.3cm}
        \begin{axis}[
            width=\columnwidth, height=2cm,
            hide axis,
            xmin=0, xmax=0, ymin=0, ymax=0,
            legend columns=3,
            legend style={at={(0,0.0)},anchor=center}
        ]
        \addlegendimage{only marks, color=NavyBlue}
        \addlegendentry{\footnotesize{$V^1$} \hspace{0.2cm}}
        \addlegendimage{only marks, color=RoyalBlue}
        \addlegendentry{\footnotesize{$\mathcal{H}^1$} \hspace{0.2cm}}
        \addlegendimage{only marks, color=DimGray}
        \addlegendentry{\footnotesize{$\bar{\mathcal{J}}^1$}}
        \end{axis}
    \end{tikzpicture}
    \begin{tikzpicture}
      \begin{axis}[
            width=\columnwidth+1cm, height=3cm,
            xmin=-1, xmax=50,
            ymin=-1.0, ymax=1.1,
            xtick distance=10,
            ytick distance=1,
            enlargelimits=false,
            ticklabel style={font=\footnotesize},
            every axis plot/.append style={thick},
            grid=none,
        ]
    
        \addplot[
          NavyBlue,
          very thick,
          name path=avg_v3
        ] table [x=x, y=avg_v3] {exp/mpe_history_peak_jsd.dat};
        
        \addplot[
          RoyalBlue,
          very thick,
          name path=avg_p3
        ] table [x=x, y=avg_p3] {exp/mpe_history_peak_jsd.dat};

        \addplot[
          DimGray,
          very thick,
          name path=avg_j-opp3
        ] table [x=x, y=avg_j-opp3] {exp/mpe_history_peak_jsd.dat};
      \end{axis}
    \end{tikzpicture}
    \caption{Episode history of agent 1.}
    \label{fig:mpe_history}
\end{subfigure}
\vspace{0.5cm}
\caption{Comparing normalized value $V^1$, $\mathcal{H}^1$ and $\bar{\mathcal{J}}^1$ on the right plot along y-axis over time steps $t$ along x-axis of an MPE prey.}
\label{fig:mpe_history_both}
\vspace{0.7cm}
\end{figure}

We measured the ICVs over $M=500$ episodes and normalized them using $\kappa=0.08$. Due to game symmetry, we focus on a single adversary (predator, denoted as \emph{adv}) and a single agent (prey, denoted as \emph{ag}). For both, we partition their coalitions and evaluate the effects on teammates (-t) and opponents (-o), where $\nu_c$ is used for teammates, while dissimilarity $\nu_d$ is applied to opponents. For predators, this means: in adv-t, $\nu_c$ captures similarity among predators and $\nu_d$ measures dissimilarity between predators and prey. In adv-o, $\nu_c$ reflects similarity among prey, and $\nu_d$ their dissimilarity with predators. The same interpretation applies to prey (ag).

Figure~\ref{fig:mpe_results} presents the results, where proximity to others is the main effect. The prey's contributions align with our expectations: average value effects $\hat{\Phi}(\nu_v)$ correspond to action determinism $\hat{\Phi}(\nu_p)$ and alignment with co-player's preferences (ag-t, $\hat{\Phi}(\nu_c)$), but interestingly, they contribute to less diversity toward predators $\hat{\Phi}(\nu_d)$. In contrast, predators benefit from reducing action determinism and keeping their choices open when engaging prey, as indicated by increased $\hat{\Phi}(\nu_c)$ but decreased $\hat{\Phi}(\nu_p)$. They tend to behave slightly different (adv-t, $\hat{\Phi}(\nu_c)$) while also promoting diversity against opponents ($\hat{\Phi}(\nu_d)$). In adv-o, high $\hat{\Phi}(\nu_p)$ and low $\hat{\Phi}(\nu_v)$ indicate prey become more deterministic (e.g., fleeing) as predators approach. Their effect also renders prey reacting differently (adv-o, $\hat{\Phi}(\nu_c)$). Overall, compared to Fig.~\ref{fig:mpe_history}, the ICVs suggest that $\hat{\Phi}(\nu_p)$ largely behaves as expected but may signal both positive and negative contributions.

\begin{figure}[htb!]
    \centering
    \hspace{-0.45cm}
    \begin{subfigure}{0.9\columnwidth}
        \centering
        \begin{tikzpicture}
            \begin{axis}[
                width=\columnwidth+0.3cm,
                height=3.5cm,
                enlarge x limits=0.2,
                xtick style={/pgfplots/major tick length=0pt},
                xtick={1, 2,3,4},
                xticklabels={adv-t, adv-o, ag-t, ag-o},
                ybar,
                bar width=7pt,
                ymin=-1, ymax=1,
                ytick distance=0.5,
                ticklabel style={font=\footnotesize},
                xticklabel style={yshift=4pt},
            ]
            \draw[solid, DarkSlateGray!50] (axis cs:1.5,-1) -- (axis cs:1.5,1);
            \draw[solid, DarkSlateGray!50] (axis cs:2.5,-1) -- (axis cs:2.5,1);
            \draw[solid, DarkSlateGray!50] (axis cs:3.5,-1) -- (axis cs:3.5,1);
            
            \addplot[fill=NavyBlue , draw= NavyBlue]
                coordinates {(1,0.35) (2,-0.95) (3, 0.8) (4, -0.2)};
            \addplot[fill=RoyalBlue , draw= RoyalBlue]
                coordinates {(1,-0.55) (2,0.75) (3, 0.85) (4, -0.08)};
            \addplot[fill=SteelBlue!80, draw=SteelBlue!80]
                coordinates {(1, -0.15) (2,-0.5) (3, 0.6) (4, 0.12)};
            \addplot[fill=DimGray, draw=DimGray]
                coordinates {(1,0.45) (2,0.27) (3, -0.4) (4, -0.35)};
            \end{axis}
        \end{tikzpicture}
    \end{subfigure}
    \hspace{-0.3cm}
    \begin{subfigure}{0.1\columnwidth}
        \centering
            \begin{tikzpicture}
                \begin{axis}[
                    width=2cm, height=2cm,
                    hide axis,
                    xmin=0, xmax=0, ymin=0, ymax=0,
                    legend columns=1,
                    legend style={at={(0,0.0)},anchor=center}
                ]
                \addlegendimage{only marks, color=NavyBlue}
                \addlegendentry{\footnotesize{$\hat{\Phi}(\nu_v)$}}
                \addlegendimage{only marks, color=RoyalBlue}
                \addlegendentry{\footnotesize{$\hat{\Phi}(\nu_p)$}}
                \addlegendimage{only marks, color=SteelBlue!80}
                \addlegendentry{\footnotesize{$\hat{\Phi}(\nu_{c})$}}
                \addlegendimage{only marks, color=DimGray}
                \addlegendentry{\footnotesize{$\hat{\Phi}(\nu_{d})$}}
                \end{axis}
            \end{tikzpicture}
            \vspace{0.05cm}
      \end{subfigure}
    \vspace{0.1cm}
    \caption{ICV $\hat{\Phi}(\cdot)$ on MPE (y-axis) showing the effect of the respective player on team members (-t) or opponents (-o) (x-axis).}
    \label{fig:mpe_results}
\vspace{0.8cm}
\end{figure}
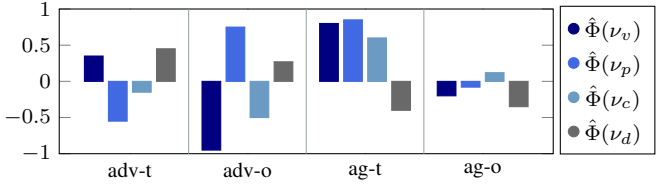

\subsection{Competition Game}\label{exp:grf_exp}
In the competitive setting, we analyze agent behavior in \emph{Google Research Football} (GRF, Fig.~\ref{fig:grf_env})~\citep{grf_env}, a complex game simulating real-world football. We employ TiZero~\citep{tizero}, a strong model, where the policy $\pi_{tz}$ is shared across agents and differentiated only by the agent index, effectively inducing \emph{heterogeneous} roles. A shared centralized critic is used, and ICV computation is performed \emph{offline}.

\begin{figure}[htb!]
    \centering
    \begin{subfigure}{0.49\columnwidth}
    \centering
        \includegraphics[width=\columnwidth-0.9cm,height=2cm]{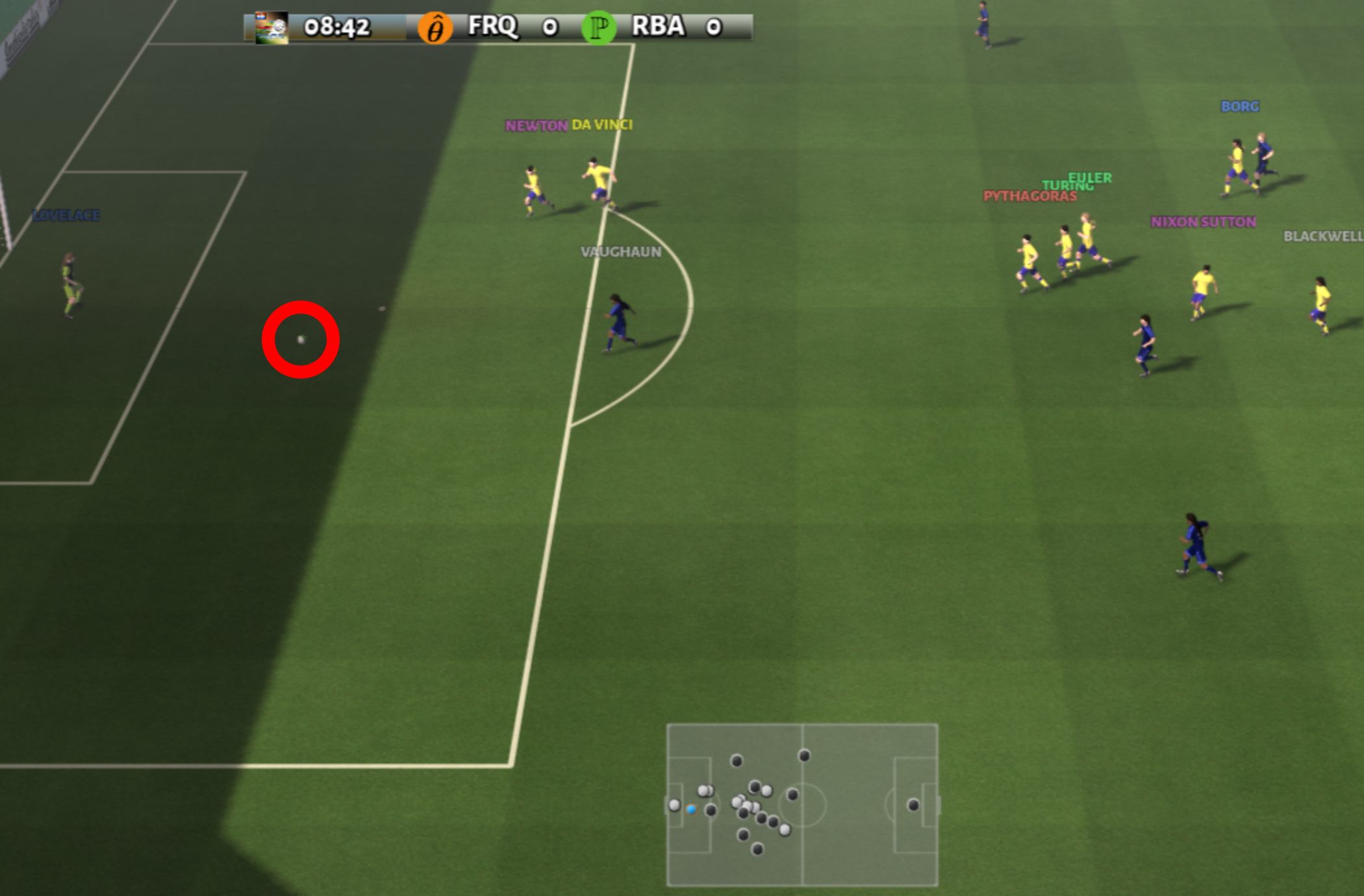}   
        \caption{Opponent scoring goal.}
        \label{fig:grf_opp_goal}
    \end{subfigure}
    \begin{subfigure}{0.49\columnwidth}
    \centering
        \includegraphics[width=\columnwidth-0.9cm,height=2cm]{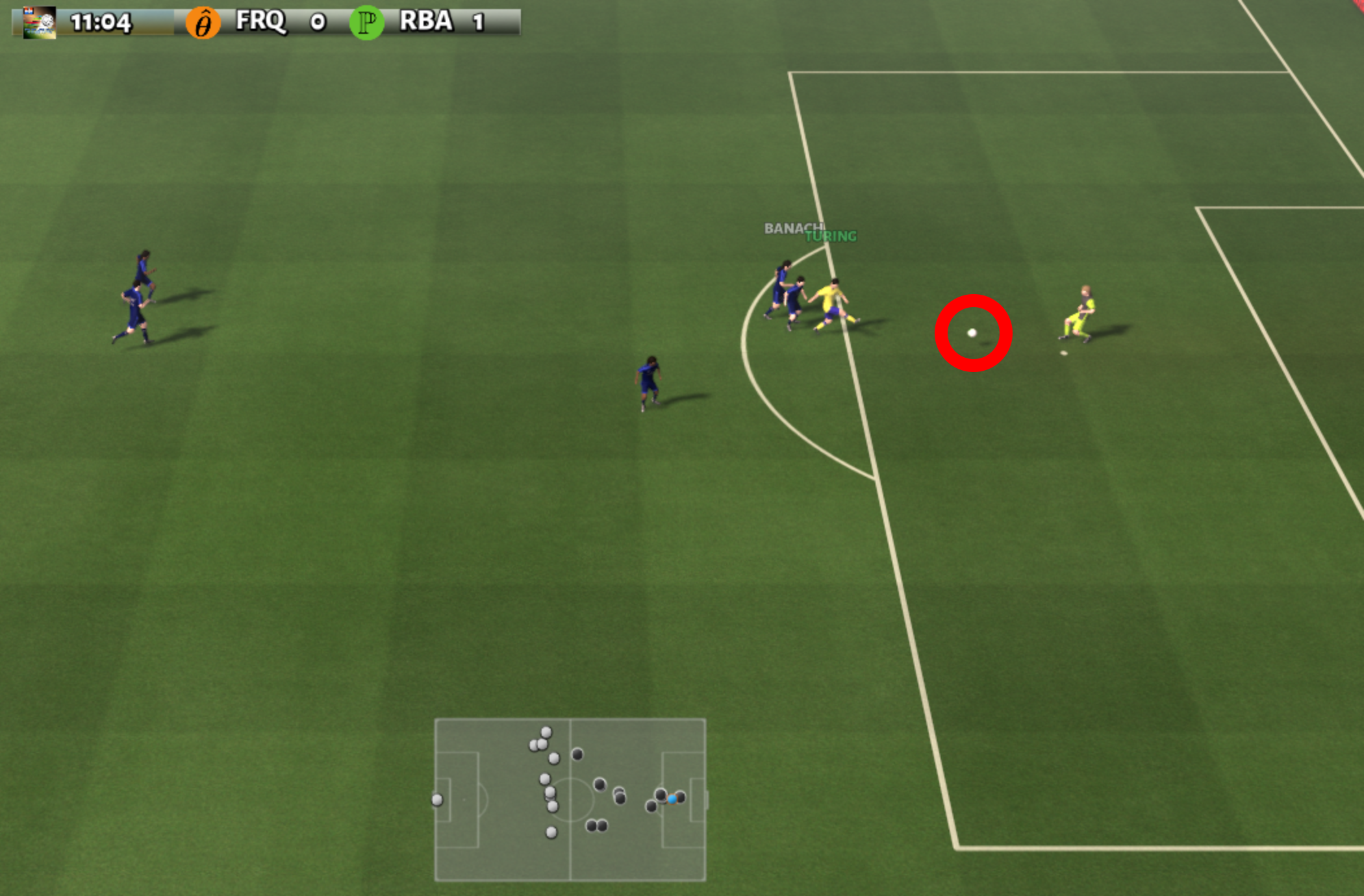}   
        \caption{Team scoring goal.}
        \label{fig:grf_team_goal}
    \end{subfigure}
    \vspace{0.5cm} 
    \caption{Two goal scenarios in GRF with ball marked in red.}
    \label{fig:grf_env}
    \vspace{0.7cm}
\end{figure} 

Our approach is flexible to measure the effects of specific groups. We define a fixed order $\sigma = (\boldsymbol{p}_{o}, \boldsymbol{p}_{t}, p_{s})$, with $\boldsymbol{p}_{o}$ opponents, $\boldsymbol{p}_{t}$ teammates, and $p_{s}$ the striker or ball player. Groups $\boldsymbol{p}_{o}$ and $\boldsymbol{p}_{t}$ are processed jointly, and ICVs are computed at intervals of $\Delta t = 2$ (Fig.~\ref{fig:MG_to_SVMG}). We measure $\bar{\nu}_d = \bar{\mathcal{J}}(\pi_{tz}(\cdot \mid s_{t,(k+1)}) \,\|\, \pi_{tz}(\cdot \mid s_{t,(k)}))$ as a proxy for strategy change, alongside $\mathcal{H}$ to capture decision certainty. The contributions over virtual time $\Tilde{t}$ $[\mathrm{min}]$ with seed=$175$ are shown in Fig.~\ref{fig:grf_history_ball}. We highlight three regions, where the situation can be reconstructed to investigate the cause. In the gray area, the team loses and regains possession near the penalty area, causing uncertainties and strategy adaption. In the orange region, the opponents scored a goal as reflected by a drop in $\hat{\Phi}(\nu_v)$ and a rise in $\hat{\Phi}(\bar{\nu}_d)$, indicating a strong strategy adaptation (scene in Fig.~\ref{fig:grf_opp_goal}). In the green region, opponents exert a negative impact on the ball player’s value and determinism just before it scored a goal (Fig.~\ref{fig:grf_team_goal}). In contrast, teammates contribute positively by increasing the ball player’s decisiveness to shoot. While $\hat{\Phi}(\nu_p)$ and $\hat{\Phi}(\bar{\nu_d})$ do not strictly match to $\hat{\Phi}(\nu_v)$, they may still provide meaningful information on players' contributions.

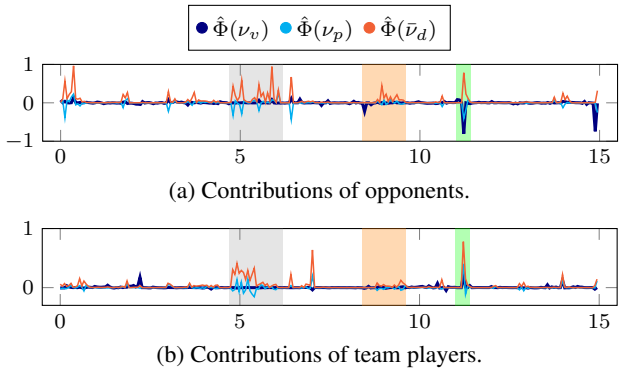
\begin{figure}[htb!]
\centering
\begin{subfigure}{1.0\columnwidth}
    \centering
    \vbox{
    \begin{tikzpicture}
        \begin{axis}[
            width=\columnwidth, height=2cm,
            hide axis,
            xmin=0, xmax=0, ymin=0, ymax=0,
            legend columns=3,
            legend style={at={(0.52,0.0)},anchor=center}
        ]
        \addlegendimage{only marks, color=NavyBlue}
        \addlegendentry{\footnotesize{$\hat{\Phi}(\nu_v)$ \hspace{0.2cm}}}
        \addlegendimage{only marks, color=cyan}
        \addlegendentry{\footnotesize{$\hat{\Phi}(\nu_p)$ \hspace{0.2cm}}}
        \addlegendimage{only marks, color=RedOrange}
        \addlegendentry{\footnotesize{$\hat{\Phi}(\bar{\nu}_{d})$}}
        \end{axis}
    \end{tikzpicture}
    }
    \vspace{-0.05cm}
    \vbox{
    \hspace{-0.38cm}
    \begin{tikzpicture}
      \begin{axis}[
            width=\columnwidth+0.5cm, height=2.6cm,
            xmin=-0.5, xmax=15.5,
            ymin=-1.0, ymax=1.0,
            xtick distance=5,
            ytick distance=1.0,
            enlargelimits=false,
            ticklabel style={font=\footnotesize},
            every axis plot/.append style={thick},
            grid=none,
        ]

        \draw[fill=gray, fill opacity=0.2, draw=none]
        (axis cs:4.7,-1.0) rectangle (axis cs:6.2,1);

        \draw[fill=orange, fill opacity=0.3, draw=none]
        (axis cs:8.4,-1.0) rectangle (axis cs:9.6,1);

        \draw[fill=green, fill opacity=0.3, draw=none]
        (axis cs:11.0,-1.0) rectangle (axis cs:11.4,1);
    
        \addplot[
          NavyBlue,
          line width=0.4mm,
          name path=v_o
        ] table [x=t, y=v_o] {exp/grf_history_ball3.dat};
        
        \addplot[
          cyan,
          line width=0.2mm,
          name path=p_o
        ] table [x=t, y=p_o] {exp/grf_history_ball3.dat};

        \addplot[
          RedOrange,
          line width=0.2mm,
          name path=j_o
        ] table [x=t, y=j_o] {exp/grf_history_ball3.dat};
        
      \end{axis}
    \end{tikzpicture}
    }
    \vspace{-0.1cm}
    \caption{Contributions of opponents.}
    \label{fig:grf_history_opp}
\end{subfigure}
\begin{subfigure}{1.0\columnwidth}
    \centering
    \vspace{0.35cm}
    \begin{tikzpicture}
      \begin{axis}[
            width=\columnwidth+0.5cm, height=2.6cm,
            xmin=-0.5, xmax=15.5,
            ymin=-0.3, ymax=1.0,
            xtick distance=5,
            ytick distance=1.0,
            enlargelimits=false,
            ticklabel style={font=\footnotesize},
            every axis plot/.append style={thick},
            grid=none,
        ]

        \draw[fill=gray, fill opacity=0.2, draw=none]
        (axis cs:4.7,-1.0) rectangle (axis cs:6.2,1);

        \draw[fill=orange, fill opacity=0.3, draw=none]
        (axis cs:8.4,-1.0) rectangle (axis cs:9.6,1);

        \draw[fill=green, fill opacity=0.3, draw=none]
        (axis cs:11.0,-1.0) rectangle (axis cs:11.4,1);
    
        \addplot[
          NavyBlue,
          line width=0.4mm,
          name path=v_c
        ] table [x=t, y=v_c] {exp/grf_history_ball3.dat};
        
        \addplot[
          cyan,
          line width=0.2mm,
          name path=p_c
        ] table [x=t, y=p_c] {exp/grf_history_ball3.dat};

        \addplot[
          RedOrange,
          line width=0.2mm,
          name path=j_c
        ] table [x=t, y=j_c] {exp/grf_history_ball3.dat};
        
      \end{axis}
    \end{tikzpicture}
    \vspace{-0.1cm}
    \caption{Contributions of team players.}
    \label{fig:grf_history_coal}
\end{subfigure}
\vspace{0.15cm}
\caption{GRF history effects on ball player or striker.}
\label{fig:grf_history_ball}
\vspace{0.7cm}
\end{figure}

Given the long GRF time horizon ($T=3000$), we compute ICV values over $M=50$ episodes, $\Delta t = 5$ steps, and normalize them by $\kappa = 0.04$. We analyze the contributions of opponents and teammates to the attacking trio, i.e., \emph{right midfield} (rm), \emph{center forward} (cf), and \emph{left midfield} (lf), as shown in Fig.~\ref{fig:grf_opps_on_trio}, and to the ball player only in Fig.~\ref{fig:grf_opps_on_ball}. We decompose the peak contributions $\hat{\Phi}(\nu_{p})$ to the trio by $\mathcal{H}^{rm}$, $\mathcal{H}^{cf}$, and $\mathcal{H}^{lm}$. We observe that credit got assigned to contributions increasing or decreasing decision certainty that again correlates to value effects. The decrease in $\mathcal{H}^{rm}$ by opponents and increase in $\mathcal{H}^{lm}$ by teammates suggest slight insecurities among the right-side players. Interestingly, none of the players exhibited a noticeable average impact on the ball player’s performance $\hat{\Phi}(\nu_v)$ and decisions $\hat{\Phi}(\nu_p)$, indicating a robust behavior to disturbances by frequently adapting its strategy, as reflected in high values of $\hat{\Phi}(\bar{\nu}_{d})$.

\begin{figure}[htb!]
    \centering
    \begin{subfigure}{0.55\columnwidth}
        \centering
        \vbox{
        \begin{tikzpicture}
        \hspace{0.35cm}
            \begin{axis}[
                width=\columnwidth, height=2cm,
                hide axis,
                xmin=0, xmax=0, ymin=0, ymax=0,
                legend columns=4,
                legend style={at={(0.8,0.0)},anchor=center}
            ]
            \addlegendimage{only marks, color=NavyBlue}
            \addlegendentry{\footnotesize{$\hat{\Phi}(\nu_v)$} \hspace{0.2cm}}
            \addlegendimage{only marks, color=RoyalBlue}
            \addlegendentry{\footnotesize{$\mathcal{H}^{rm}$} \hspace{0.2cm}}
            \addlegendimage{only marks, color= SteelBlue!80}
            \addlegendentry{\footnotesize{$\mathcal{H}^{cf}$} \hspace{0.2cm}}
            \addlegendimage{only marks, color= DimGray}
            \addlegendentry{\footnotesize{$\mathcal{H}^{lm}$}}
            \end{axis}
        \end{tikzpicture}
        }
        \vspace{-0.05cm}
        \vbox{
        \begin{tikzpicture}
            \begin{axis}[
                width=\columnwidth+0.4cm,
                height=3.5cm,
                enlarge x limits=0.55,
                xtick style={/pgfplots/major tick length=0pt},
                xtick={1, 2},
                xticklabels={Opp, Team},
                ybar,
                bar width=7pt,
                ymin=-1, ymax=1,
                ytick distance=0.5,
                ticklabel style={font=\footnotesize},
                xticklabel style={yshift=4pt},
            ]
            \draw[solid, DarkSlateGray!50] (axis cs:1.5,-1) -- (axis cs:1.5,1);
            
            \addplot[fill=NavyBlue , draw= NavyBlue]
                coordinates {(1,-0.3) (2,0.8)};
            \addplot[fill= RoyalBlue, draw= RoyalBlue]
                coordinates {(1,-0.95) (2,0.2)};
            \addplot[fill=SteelBlue!80 , draw=SteelBlue!80]
                coordinates {(1,-0.4) (2, 0.35)};
            \addplot[fill=DimGray, draw=DimGray]
                coordinates {(1,0.-0.1) (2,0.6)};
            \end{axis}
        \end{tikzpicture}
        \vspace{-0.05cm}
        \caption{Effect on forwarding players.}
        \label{fig:grf_opps_on_trio}
        }
    \end{subfigure}
    \hspace{-0.2cm}
    \begin{subfigure}{0.45\columnwidth}
        \centering
        \vbox{
        \begin{tikzpicture}
        \hspace{0.36cm}
            \begin{axis}[
                width=\columnwidth, height=2cm,
                hide axis,
                xmin=0, xmax=0, ymin=0, ymax=0,
                legend columns=3,
                legend style={at={(0,0.0)},anchor=center}
            ]
            \addlegendimage{only marks, color=NavyBlue}
            \addlegendentry{\footnotesize{$\hat{\Phi}(\nu_v)$} \hspace{0.2cm}}
            \addlegendimage{only marks, color=RoyalBlue}
            \addlegendentry{\footnotesize{$\hat{\Phi}(\nu_p)$} \hspace{0.2cm}}
            \addlegendimage{only marks, color= SteelBlue!80}
            \addlegendentry{\footnotesize{$\hat{\Phi}(\bar{\nu}_{d})$}}
            \end{axis}
        \end{tikzpicture}
        }
        \vspace{-0.05cm}
        \vbox{
        \begin{tikzpicture}
            \begin{axis}[
                width=\columnwidth+0.5cm,
                height=3.5cm,
                enlarge x limits=0.55,
                xtick style={/pgfplots/major tick length=0pt},
                xtick={1, 2, 3},
                xticklabels={Opp, Team},
                ybar,
                bar width=7pt,
                ymin=-1, ymax=1,
                ytick distance=0.5,
                ticklabel style={font=\footnotesize},
                xticklabel style={yshift=4pt},
            ]
            \draw[solid, DarkSlateGray!50] (axis cs:1.5,-1) -- (axis cs:1.5,1);
            
            \addplot[fill=NavyBlue , draw= NavyBlue]
                coordinates {(1,-0.04) (2,0.06)};
            \addplot[fill= RoyalBlue, draw= RoyalBlue]
                coordinates {(1,-0.03) (2,0.05)};
            \addplot[fill= SteelBlue!80, draw= SteelBlue!80]
                coordinates {(1,0.95) (2,0.6)};
            \end{axis}
        \end{tikzpicture}
        \vspace{-0.05cm}
        \caption{Effect on ball player.}
        \label{fig:grf_opps_on_ball}
        }
    \end{subfigure}
    \vspace{0.1cm}
    \caption{ICVs in GRF of opponents and team players on forwarding and ball player in terms of value, determinism and strategy change.}
    \label{fig:grf_results}
    \vspace{0.4cm}
\end{figure}
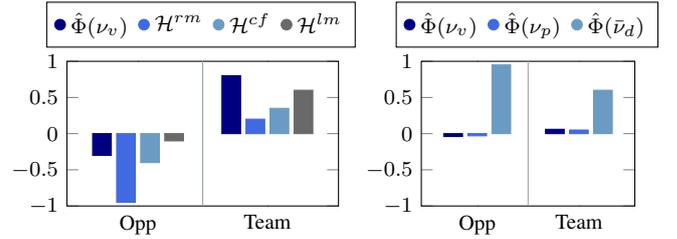

\section{Conclusions}\label{sec:Conclusion}
This paper introduced \emph{Intended Cooperation Values}~(ICVs), a novel causal action attribution method that explains agent behavior through policy influence rather than reward-based evaluations. ICVs quantify how individual actions affect teammates’ decision certainty and preference alignment, effectively capturing implicit cooperation and consensus. The approach proved effective across a variety of environments (cooperative, competitive, and mixed-motive). A major advantage of ICV lies in its simplicity and generality to operate without requiring reward access or architectural changes, and it can be applied both online and offline. This makes it well-suited for scenarios where interpretability is critical but value functions are unavailable or unreliable. In such cases, any form of domain knowledge (e.g., being determined to a specific action, as the \emph{load} action in LBF) may serve as a heuristic to evaluate situations and correctly interpret ICV results. Furthermore, per-step histories of ICVs allow more fine-grained analysis to extract behavioral insights when value feedback is missing, without requiring the situation to be rendered. This is particularly well suited for complex dynamic tasks like GRF. Thus, ICVs offer a principled way to analyze intentions and influences, grounded in game-theoretic and information-theoretic principles. However, the method also has limitations. Computing ICVs in complex, large-scale environments requires considerable post-processing due to the need for sequential action decomposition and intermediate state reconstruction. In addition, our method assumed full state observability, which may limit applicability in some settings. Nonetheless, both SVMG and ICV can still be applied without modification in partially observable domains. However, the resulting ICVs may be less accurate, as some action effects may not be visible to all agents and thus may not alter their policy distributions. More generally, Shapley-based methods inherently involve averaging over all marginal contributions, which can reduce accuracy even in fully observable settings.



\clearpage

\begin{ack}
We gratefully acknowledge the authors of TiZero~\cite{tizero} for providing access to their full model, which was essential for our experiments.
\end{ack}

\section*{Ethics Statement} Our method is intended to enhance transparency and safety in MARL systems and does not incorporate any form of adversarial attack and pose no risk to human safety. All experiments were conducted in a simulated environment, without raising ethical or fairness concerns.

\bibliography{refs}

\clearpage

\appendix
\onecolumn

\begin{center}
    {\LARGE \textbf{Supplementary Material}}
\end{center}

\section{Shapley Value Axioms}\label{app:shapley_axioms}

\begin{enumerate}
    \item \textbf{Efficiency:} The sum of the SVs for all players equals the total value of the grand coalition:
    \begin{equation}
        \sum_{i=1}^{n} \phi_i(v) = v(N)\,.
        \label{eq:shap_axiom_efficiency}
    \end{equation}

    \item \textbf{Symmetry:}  If two players are interchangeable (i.e., they contribute equally to every possible coalition),
    \begin{equation}
        v(C \cup \{i\}) = v(C \cup \{j\}) \quad \text{for all  } C \subseteq N \setminus \{i,j\}\,,
        \label{eq:shap_axiom_symmetry}
    \end{equation}
    then they receive the same SV $\phi_i(v) = \phi_j(v)$.

    \item \textbf{Additivity:} For any two games $v$ and $w$, the SV of their sum $v+w$ is the sum of their SVs:
    \begin{equation}
        \phi_i(v + w) = \phi_i(v) + \phi_i(w) \quad \text{for all } i\,.
        \label{eq:shap_axioms_additivity}
    \end{equation}
    
    \item \textbf{Dummy Player:} If a player does not contribute any additional value to any coalition $C \subseteq N$,
    \begin{equation}
        v(C \cup \{i\}) = v(C)\,,
        \label{eq:shap_axioms_dummy}
    \end{equation}
    then its SV is zero $\phi_i(v) = 0$.
\end{enumerate}

\section{Propositions and Proofs}\label{app:instr_proofs}

In this section, we provide the propositions in a slightly more explicit form, followed by their proofs. 

\subsection{Advantage}\label{app:adv_proof}

\paragraph{Proposition 1.} Given the value function $V^j(s)$ as defined by Eq.~\eqref{eq:value_function} that is learned under $\mathcal{G}_m$ dynamics, and employing it as a \emph{static predictor} within $\mathcal{G}_s$, then for each transition of intermediate states $s_{t,(k-1)} \rightarrow s_{t,(k)}$ induced by processing action $a_t^i$, the value-based marginal contribution, which is given as
$$\Delta \nu_v = \sum_{j \in \mathcal{C}_{\sigma}^i} \left[ V^j(s_{t,(k)}) - V^j(s_{t,(k-1)}) \right ],$$
quantifies an undiscounted sampled advantage $\hat{\Lambda}(s_{t,(k-1)}, a_t^i, s_{t, (k)})$ for players in $\mathcal{C}_{\sigma}^i$ attributable to agent $i$:
$$\Delta \nu_{v} = \sum_{j \in \mathcal{C}_{\sigma}^i} \hat{\Lambda}^j(s_{t,(k-1)}, a_t^i, s_{t, (k)}).$$

\begin{proof}

We consider the intermediate states $(s_{t,(k-1)},s_{t,(k)}) \in \mathcal{S}$ and the action $a^i_t \in \mathcal{A}^i$ of agent $i=\sigma(k)$, processed according to Eq.~\eqref{eq:inter_state_udpate_explicit}. The MG value function $V^j(\cdot)$ is evaluated \emph{statically} at concrete state samples, which are either visited or constructed according to Def.~\ref{def:inter_state}. This implies no evaluation over all possible future intermediate states. Therefore, we transition from the expectation-based advantage function $\Lambda(s_{t,(k-1)},a^i_t)$ to the sample-based advantage $\hat{\Lambda}(s_{t,(k-1)},a^i_t, s_{t, (k)})$ as in Def.~\eqref{def:advantage_func}. Following the definition of marginal contribution from Def.~\ref{def:icv_marginal}, $\hat{\Lambda}(s_{t,(k-1)},a^i_t, s_{t, (k)})$ is evaluated under the SVMG intermediate-step dynamics for each agent $j \in \mathcal{C}_{\sigma}^i$. This implies:

\begin{equation}
    \sum_{j \in \mathcal{C}_{\sigma}^i} \hat{\Lambda}(s_{t,(k-1)}, a_t^i, s_{t, (k)}) = \sum_{j \in \mathcal{C}_{\sigma}^i} \left[ R^j(s_{t,(k-1)}, a^i_t,s_{t,(k)}) + \gamma V^j(s_{t,(k)}) - V^j(s_{t,(k-1)}) \right].
    \label{eq:adv_proof1}
\end{equation}

As the SVMG framework does not include a reward function $R^j$, meaning rewards are only provided at full time-step transitions $t \rightarrow t+1$ and not at intermediate stages $k$, we have $R^j(s_{t,(k-1)}, a^i_t, s_{t,(k)}) = 0$ for all sub-step transitions. Thus, Eq.~\eqref{eq:adv_proof1} rewrites as:

\begin{equation}
    \sum_{j \in \mathcal{C}_{\sigma}^i} \hat{\Lambda}(s_{t,(k-1)}, a_t^i, s_{t, (k)}) = \sum_{j \in \mathcal{C}_{\sigma}^i} \left[ \gamma V^j(s_{t,(k)}) - V^j(s_{t,(k-1)}) \right].
    \label{eq:adv_proof2}
\end{equation}

During action processing in the SVMG chain, there is no increment of the time step $t$. Thus, we have $\gamma^0=1$, and Eq.~\eqref{eq:adv_proof2} becomes:

\begin{equation}
    \sum_{j \in \mathcal{C}_{\sigma}^i} \hat{\Lambda}(s_{t,(k-1)}, a_t^i, s_{t, (k)}) = \sum_{j \in \mathcal{C}_{\sigma}^i} \left[ V^j(s_{t,(k)}) - V^j(s_{t,(k-1)}) \right ].
    \label{eq:adv_proof3}
\end{equation}

By Defs.~\ref{def:icv_marginal} and~\ref{def:v_v}, the r.h.s. of Eq.~\eqref{eq:adv_proof3} corresponds to the value-based marginal contribution $\Delta \nu_v$. This  proves the proposition.
\end{proof}

\subsubsection{Remark}
The value function $V(s)$ is originally defined for states $s \in \mathcal{S}$ under MG dynamics, where \emph{all} agents act simultaneously at state $s$. It is learned by the agent's critic during training as an estimate of the expected return from $s$, as defined by Eq.~\eqref{eq:value_function}. However, as written in Proposition~1 and shown by~\citet{shapley_rl1}, $V(s)$ serves as a \emph{static} predictor when applying it as feature attribution \emph{without} ever evaluating the policy $\boldsymbol{\pi}$. Evaluating $V(s_{t,(k)})$ for intermediate states $s_{t,(k)} \in \mathcal{S}$ thus estimates the return \emph{as if} all agents acted simultaneously at $s_{t,(k)}$ under policy $\boldsymbol{\pi}$ and under MG dynamics. In the SVMG framework, however, the actions of all agents are fixed and do not react to such intermediate updates, as illustrated in Fig.~\ref{fig:MG_to_SVMG}. This results in a counterfactual or hypothetical evaluation: if the game were in state $s_{t,(k)}$ and all agents acted according to $\boldsymbol{\pi}$, what would the value be? Based on this reasoning, we leverage the MG-trained value function $V$ within the SVMG framework to quantify agents' intentional contributions during action processing. The undiscounted sampled advantage $\hat{\Lambda}(s_{t,(k-1)}, a_t^i, s_{t, (k)})$, when computed under the SVMG dynamics, thus corresponds to the empirical difference $V(s_{t,(k)}) - V(s_{t,(k-1)})$. This approach explicitly serves the purpose of counterfactual attribution.

It is important to note that applying the MG-trained value function $V$ in this manner generally does \emph{not} result in a Bellman-consistent advantage under SVMG dynamics. The reason is that $V(s_{t,(k)})$ does not represent the expected return if the process were initiated from $s_{t,(k)}$ and evolved according to the intermediate-step dynamics of the SVMG. Specifically, the states $s_{t,(k)}$ alone are not Markovian under SVMG dynamics when evaluated with the MG value function $V$. To fulfill the value recursion under the SVMG framework, one would need to enrich the intermediate state RVs defined in Def.~\ref{def:inter_state} by explicitly incorporating the step counter $k$ and permutation $\sigma$ into the state representation, i.e., by defining an augmented state $\tilde{S}_{t,(k)} = (k, \sigma, S_{t,(k)})$. Additionally, a separate value function $\tilde{V}(\tilde{s})$ would have to be trained or estimated explicitly under the intermediate-step transition kernels $P^{\sigma(k)}(\cdot \mid s_{t,(k)},a_t^{\sigma(k)})$ of the SVMG. For such a function $\tilde{V}$ to be learned correctly, rewards would need to be distributed at each intermediate step $k$, preserving both the Markov property and ensuring valid value recursion. Under these conditions, the difference $\hat{\Lambda}^j(s_{t,(k-1)}, a_t^i, s_{t, (k)})$ computed using $\tilde{V}(\tilde{s})$ would correspond to a Bellman update.

However, since our objective is to devise a method that does not require training an additional SVMG-specific model (see Sec.~\ref{sec:method}), we instead utilize the existing MG-trained value function $V$ directly as a numerical function. This allows us to evaluate $V$ at any encountered or reconstructed intermediate state realization $s_{t,(k)}$, provided $s_{t,(k)} \in \mathcal{S}$, thus making the computation $V(s_{t,(k)})$ both valid and practically feasible.

\subsection{Empowerment}\label{app:empower_proof}

\paragraph{Proposition 2.}
For each intermediate state $s_{t,(k)} \in \mathcal{S}$, the entropy-based marginal contribution, given as:
$$\Delta \nu_{p} = \sum_{j\in \mathcal{C}_{o}^{i}} \left[ \mathcal{H}\left( A_{t,(k+1)}^j \right) - \mathcal{H}\left( A_{t,(k)}^j \right) \right]\,,$$
measures the mutual information $\mathrm{I}(A_{t,(k)}^j ; a_t^i \mid s_{t,(k-1)})$ as the certainty increase in decision-making of players in $\mathcal{C}_{\sigma}^i$ attributable to agent $i$'s action $a_t^i$:
$$\Delta \nu_{p} = \sum_{j\in \mathcal{C}_{\sigma}^i} \mathrm{I}(A_{t,(k)}^j ; a_t^i \mid s_{t,(k-1)})\,.$$
Moreover, if the acting policy $\pi^i$ maximizes each conditional mutual information, then $\Delta \nu_{p}$ induces the \emph{instrumental empowerment}:
$$\Delta \nu_{p} =\sum_{j\in \mathcal{C}_{\sigma}^i} \mathcal{E}^j(s_{t,(k-1)}, i)\,.$$

\begin{proof}

After plugging $\nu_p$ from Def.~\ref{def:v_p} with \emph{peakedness} $\mathcal{H}$ into the marginal contribution from Def.~\ref{def:icv_marginal}, we get

\begin{align}
    \Delta \nu_{p} 
    &= \sum_{j\in \mathcal{C}_{o}^{i}} \left[ \mathcal{H}\left( A_{t,(k+1)}^j \right) - \mathcal{H}\left( A_{t,(k)}^j \right) \right] \\
    &= \sum_{j \in \mathcal{C}_{o}^{i}} \left[ (\log|\mathcal{A}^j| - H\left(A_{t,(k+1)}^j \mid s_{t,(k)}\right)) - (\log|\mathcal{A}^j| - H\left(A_{t,(k)}^j \mid s_{t,(k-1)}\right)) \right]\\
    &= \sum_{j \in \mathcal{C}_{o}^{i}} \left[ H\left(A_{t,(k)}^j \mid s_{t,(k-1)}\right) - H\left(A_{t,(k+1)}^j \mid s_{t,(k)}\right) \right]\,.
    \label{eq:proof_empower3}
\end{align}

We assume that $\pi^j(\cdot \mid s_{t,(k)}) \approx \pi^j(\cdot \mid s_{t,(k-1)}, a_t^i)$ holds (as otherwise communication-based policies would be required, which are not the scope of this work.). This implies that the state $s_{t,(k)}$ encodes sufficient information about $a_t^i$ to enable a consistent measurement of the RV $A^j_{t,(k)}$ at sub-step $k$ and state $s_{t,(k-1)}$ by \emph{observing} the action $a_t^i$, but without actually \emph{executing} it to transition to a new state $s_{t,(k)}$, as per Def.~\ref{def:inter_state}. Thus, Eq.~\eqref{eq:proof_empower3} transforms to:

\begin{equation}
    \Delta \nu_{p} = \sum_{j \in \mathcal{C}_{o}^{i}} \left[ H\left(A_{t,(k)}^j \mid s_{t,(k-1)}\right) - H\left(A_{t,(k)}^j \mid s_{t,(k-1)}, a_t^i\right) \right]\,,
    \label{eq:proof_empower4}
\end{equation}

which, by the definition of the conditional mutual information, gives:

\begin{equation}
    \Delta \nu_{p} = \sum_{j \in \mathcal{C}_{o}^{i}} \mathrm{I}\left(A_{t,(k)}^j ; a_t^i \mid s_{t,(k-1)} \right)\,.
    \label{eq:proof_empower5}
\end{equation}

The instrumental empowerment then trivially follows by Def.~\ref{def:instr_empowerment}.

\end{proof}

\end{document}